\documentclass[11pt]{article}
\usepackage{ya}

\newcommand{\E}{\mathbb E}
\newcommand{\Prb}{\mathbb P}
\newcommand{\R}{\mathbb R}
\newcommand{\one}{\mathbf 1}
\newcommand{\Ber}{\operatorname{Ber}}
\newcommand{\KL}{D_{\mathrm{KL}}}
\newcommand{\gap}{\operatorname{gap}}

\newcommand{\Delt}{\Delta}
\newcommand{\Npw}{\mathrm{N}}
\newcommand{\Nfront}{\mathrm{N}_{\mathrm{frontier}}}
\newcommand{\Naq}{\mathrm{N}_{\mathrm{ADA}}}
\newcommand{\Nsel}{\mathrm{N}_{\mathrm{selection}}}
\newcommand{\Mcap}{\mathsf{M}}
    
\title{How Reusable Are Benchmarks with Richer Feedback?}

\YATwoAuthors{Youssef Allouah}{\textit{Department of Computer Science\\Stanford University}}
             {John Duchi}{\textit{Departments of Statistics and Electrical Engineering\\Stanford University}}
\date{}

\begin{document}
\maketitle

\begin{abstract}
    We study whether benchmarks reliably guide model selection as developers adapt to evaluation feedback across multiple criteria.
    We find that the worst-case test-set size needed to estimate the best score among $k$ adaptively chosen models, under any convex combination of the criteria, grows exponentially with the number of criteria, reaching the $\Theta(\sqrt{k})$ cost of answering $k$ adaptive statistical queries with only $\mathcal O(\log k)$ criteria, at fixed accuracy and confidence.
    In attacks on multi-task large language model benchmarks with five to ten criteria, feedback restricted to nondominated task profiles produces large reused-to-held-out score gaps and frequent false winners.
    These results challenge a prominent explanation for prior observed reliable benchmark reuse—that developers mainly respond to convincing improvements over the current best—in rich-feedback settings, while leaving open how often ordinary model development encounters this vulnerability.
\end{abstract}

\vspace{0pt plus 2fill}
\setcounter{tocdepth}{2}
\etocsettagdepth{appendixstart}{none}

\YACompactContents
\vspace*{.5\baselineskip}
\clearpage

\section{Introduction}
\label{sec:introduction}

Benchmarks are a standard way to compare machine learning systems and guide model selection. For these comparisons to inform deployment, they must generalize: performance on the underlying test set should reflect performance on fresh data from the same population.

\vspace{-1mm}
\paragraph{The adaptivity paradox.}
Benchmarks are often consulted for years, with later models developed in light of earlier results. Once model choice depends on test-set feedback, the model and test set are no longer independent, so statistical concentration arguments such as Hoeffding's inequality no longer apply directly. This issue persists even with a hidden test set, without contamination through exposure to test examples during training. The adaptive data analysis line of work studies such feedback-dependent reuse and shows that accurately answering a sequence of adaptive statistical queries can require substantially more data than for its nonadaptive counterpart~\cite{hardt2014preventing,dwork2015preserving,steinke2015interactive}.

Nevertheless, empirical evidence is often more reassuring. For example, fresh-data replications of vision benchmarks found shifts in absolute accuracy while largely preserving relative model progress, and a meta-analysis of over one hundred Kaggle competitions found little evidence of severe overfitting from ordinary leaderboard reuse~\cite{recht2019imagenet,roelofs2019meta}.

Blum and Hardt offer one elegant explanation: a leaderboard need only track the best value seen so far, and not every submission's value like in adaptive data anaylsis~\cite{blum2015ladder}. Their \emph{Ladder} mechanism releases updates only after sufficiently large improvements, using a test-set size of the same order as the nonadaptive scenario at fixed accuracy. In their own words, researchers who mainly respond to convincing improvements over the previous best implicitly simulate the Ladder.

\vspace{-1mm}
\paragraph{From one ranking to many.}
Models increasingly trade off performance across tasks, safety criteria, and deployment costs, so different weightings of these criteria produce different rankings~\cite{zhang2024tradeoffs,jung2026defines}. 
With the Ladder, only a new record changes the feedback; across multiple rankings, many more submissions can become best. But, how reusable is a benchmark with many notions of best?

To answer this question, we formalize a leaderboard abstraction using multiple evaluation criteria. It also specifies a family of criteria weight vectors modeling the tradeoffs that the benchmark is asked to validate. At each round, a developer submits a model and a user requests a weighting, both of which may depend on earlier feedback. Then, using a fixed hidden test set, the evaluator reports only an estimate of the best population score among all models submitted so far, under the requested weighting.
Richer feedback releases, like task-level evaluations or Pareto frontiers, are common in practice and can answer these weighted requests by postprocessing, so our negative results also apply to them.

\vspace{-2mm}
\subsection{Main results}

Our main result quantifies a phase transition in the test-set size required for statistically valid reuse.
\begin{theorem}[Phase transition, informal]
Fix sufficiently small target error and failure probability. Let $n_\star$ be the smallest test-set size for which every benchmark with $d$ criteria admits a mechanism that accurately tracks all requested best-so-far values over $k$ adaptive submissions under all convex weightings. With  constant $C>0$,
\begin{equation*}
    n_\star =
\begin{cases}
\Theta_d(\log k), &\text{for every fixed }d,\\[2pt]
\Theta(\sqrt k), &\text{when }d\ge C\log k.
\end{cases}
\end{equation*}
\end{theorem}
Therefore fixed dimension retains logarithmic dependence on the number of submissions as in the absence of adaptivity, while only logarithmically many criteria suffice to reach the $\Theta(\sqrt{k})$ sample size scale of general adaptive estimation~\cite{steinke2015interactive}. Between these regimes, the adaptive cost grows exponentially in $d$, up to constants in the exponent and logarithmic factors.
Surprisignly, even returning only an approximately best model's identity retains the high-dimensional $\Theta(\sqrt{k})$ cost.

For a general family of weightings, the sample size is governed by \emph{frontier capacity}: the length of the largest evaluation profiles sequence that can successively become best by a fixed margin, each under some weighting in the family. Our lower bound uses frontier capacity to establish a reduction from general adaptive estimation: the aforementioned sequence makes the leaderboard leak sample-dependent information at every round, although nominally it is only asked to track the best value for requested weightings. Conversely, our \emph{Frontier Ladder} algorithm retains profiles only when they sufficiently improve the best value under some weighting. Frontier capacity bounds the number of these releases; combined with a differentially private mechanism, this gives the matching upper bounds above. 
\paragraph{Empirical evidence.}
Table~\ref{tab:llm-summary} summarizes controlled attacks on four multi-task large language model (LLM) benchmark snapshots. Our attack submits random \emph{prompt routers} that choose between two well-chosen LLMs, then use the released feedback once to construct a final router via majority voting of the routers, similar to the boosting attack of Blum and Hardt~\cite{blum2015ladder}. Scalar feedback reveals only new aggregate records; Pareto feedback reveals task profiles of routers not dominated by earlier submissions.

\begin{table}[t]
\centering
\small
\setlength{\tabcolsep}{8pt}
\renewcommand{\arraystretch}{1.1}
\newcommand{\benchmarkcite}[1]{{\footnotesize\color{black!65}\cite{#1}}}
\begin{tabular}{@{}l@{\enspace}rcccc@{}}
\toprule
\multicolumn{2}{@{}l}{\textbf{Benchmark}}
& \multicolumn{2}{c}{\shortstack{\textbf{Leaderboard score gap}\\{\footnotesize (90\textsuperscript{th} percentile, pp)}}}
& \multicolumn{2}{c}{\shortstack{\textbf{False-winner rate}\\{\footnotesize (1\textsuperscript{st} on test set, not on held-out)}}} \\
\cmidrule(lr){3-4}\cmidrule(lr){5-6}
& & \makebox[5.5em]{Pareto} & \makebox[5.5em]{Scalar}
& \makebox[5.5em]{Pareto} & \makebox[5.5em]{Scalar} \\
\midrule
HELM Capabilities & \benchmarkcite{liang2023holistic} & 6.36 & 3.86 & 0.932 & 0.560 \\
LiveBench & \benchmarkcite{livebench} & 10.54 & 7.22 & 0.884 & 0.760 \\
HELM Lite & \benchmarkcite{liang2023holistic} & 4.62 & 2.00 & 0.992 & 0.008 \\
Open LLM & \benchmarkcite{open-llm-leaderboard} & 11.93 & 3.48 & 1.000 & 0.612 \\
\midrule
Average & & 8.36 & 4.14 & 0.952 & 0.485 \\
\bottomrule
\end{tabular}
\caption{\textbf{Our attack undermines best-score generalization and best-model selection.}
The attacker observes only restricted leaderboard feedback:
(i)~\emph{Pareto} feedback reveals only task profiles not dominated by earlier
submissions; (ii)~\emph{Scalar} feedback reveals only new task-aggregate records.
Gaps are $90$th percentiles of the maximum absolute reused-to-held-out
difference in best-so-far task-aggregate loss (pp denotes percentage points).
False-winner rate is the fraction of all $250$ trials where  the attacker
ranks first on the reused split but not on the held-out  against original models.
Attack budgets are $2{,}048$ submissions for
LiveBench and $32{,}768$ otherwise. More details in Section~\ref{sec:experiments}.}
\label{tab:llm-summary}
\end{table}

Even with such restrictions, Pareto feedback consistently yields a larger leaderboard score gap and a higher false-winner rate. This establishes a vulnerability under multi-criteria feedback, and leaves open how often ordinary model development encounters it. Section~\ref{sec:experiments} describes the attack, and Figure~\ref{fig:real-budget-results} tracks leaderboard gaps and false-winner rates across attack budgets. Appendix~\ref{sec:llm-supplement} reports additional results.

\paragraph{Broader implications.}
Our results put a limit on a familiar explanation for observed reliable benchmark reuse: that responding mainly to clear improvements over the previous best protects against overfitting~\cite{blum2015ladder,hardt2026emerging}. This protection can weaken when there are many notions of best, so reassuring experience with scalar leaderboards need not carry over to multi-criteria evaluation. In addition to theoretical limits, our simple attacks on real benchmark snapshots give a concrete reason to re-examine that expectation. Moreover, the current high commercial stakes of LLM rankings give developers incentives to game evaluation~\cite{singh2025leaderboard}. Will ordinary model development nevertheless remain reliable under richer feedback? If so, can properties of real model families and feedback use explain why, perhaps due to small realized frontier capacity? If not, can we design mechanisms that address this vulnerability while preserving useful comparisons at realistic sample sizes and computational costs?

\clearpage
\section{Benchmark Evaluation with Multiple Criteria}
\label{sec:model}

Let $\mathcal Z$ be the space of benchmark data and $P$ a probability distribution on $\mathcal Z$. For instance, $Z=(X,Y)\sim P$ can consist of an input $X$ and label $Y$. A benchmark fixes $d$ evaluation criteria, a class $\mathcal F$ of models, and an \emph{evaluation map}
$\phi:\mathcal F\times\mathcal Z\longrightarrow[0,1]^d$, with smaller values indicating better performance.
For a model $f\in\mathcal F$, its \emph{population evaluation profile} $R_P(f)$ records the population value of each criterion:
\begin{equation}
    R_P(f)\coloneqq\E_{Z\sim P}[\phi(f,Z)]\in[0,1]^d.
\end{equation}

\begin{example}[Multi-task benchmarks]
The coordinates can represent prediction losses on different benchmark tasks, e.g., see the HELM benchmark\footnote{Example of a real multi-task benchmark: \url{https://crfm.stanford.edu/2025/03/20/helm-capabilities.html}.}.
Let task $j$ have population $P_j$ and prediction loss $\ell_j$.
With $P=\bigotimes_{j=1}^d P_j$, write
$Z=(Z_1,\ldots,Z_d)\sim P$ and define
$\phi_j(f,Z):=\ell_j(f,Z_j).$
Then $R_{P,j}(f)=\mathbb E_{P_j}[\ell_j(f,Z_j)]$ is the population loss
on task $j$. Here $n$ observations correspond to $n$ independently
sampled examples per task.
\end{example}

The evaluation interface specifies a \emph{weight family} $\mathcal W$: the weight vectors that users may request when comparing models. Formally, $\mathcal W$ is a nonempty compact subset of the simplex
$\Delta_d\coloneqq\{w\in\mathbb R_+^d:\langle w,\mathbf 1\rangle=1\},
$
where $\mathbf 1$ is the vector of $d$ ones. For $w\in\mathcal W$, the aggregate risk of $f$ is
$\langle w,R_P(f)\rangle$.

Each $w\in\mathcal W$ induces an ordering of the submitted models, and our
validity criterion tracks only the best value under that ordering.  A singleton
$\mathcal W$ therefore reduces to the scalar leaderboard problem, regardless
of $d$, while $\mathcal W=\Delta_d$ allows any convex combination of the
evaluation coordinates and is our main case.  Explicit support for multiple aggregations
has also been implemented in evaluation interfaces: prior work allowed users to
reweight evaluation metrics, and recent work prototypes an interactive LLM
leaderboard in which users reweight prompt slices and recompute
rankings~\cite{ma2021dynaboard,jung2026defines}.

\subsection{The weighted leaderboard problem}
\label{sec:weighted-leaderboard}

Fix the earlier benchmark design $(\mathcal Z,\mathcal F,\phi)$, integers $n,k \geq 1$, a population law $P$ on $\mathcal Z$ unknown to the mechanism, and a hidden i.i.d. sample
$S\coloneqq(Z_1,\ldots,Z_n)\sim P^n$.
At round $t\in[k]\coloneqq\{1,\ldots,k\}$, an \emph{analyst} submits a model $f_t\in\mathcal F$ and requests a weighting $w_t\in\mathcal W$. This analyst represents the joint choices of model developers and users requesting comparisons. Both choices may depend arbitrarily on the preceding public interaction transcript. 

After submissions $f_1,\ldots,f_t$, the best population value under weight vector $w \in \mathcal W$ is
\begin{equation}
    V_t^P(w)\coloneqq\min_{s\le t}\langle w,R_P(f_s)\rangle.
    \label{eq:population-envelope}
\end{equation}
A (randomized) mechanism $M$ accessing $S$ receives $(f_t,w_t)$, then returns $\widehat v_t\in\R$ to estimate $V_t^P(w_t)$.

\begin{definition}[Weighted leaderboard validity]
\label{def:weighted-validity}
For fixed $(\mathcal Z,\mathcal F,\phi)$, $\mathcal W$, $n$, $k$, and $0<\varepsilon,\beta<1$, a mechanism is $(\varepsilon,\beta)$-valid over $\mathcal W$ with $n$ samples and $k$ submissions if, for every population law $P$ on $\mathcal Z$ and every adaptive analyst making at most $k$ submissions, when run on $S\sim P^n$, we have
\begin{equation*}
    \Prb\!\left[
      \max_{t\le k}
      |\widehat v_t-V_t^P(w_t)|
      \le \varepsilon
    \right]
    \ge 1-\beta,
\end{equation*}
where the probability is over the sample $S$ and the mechanism and analyst.
Let $\Npw(k,\mathcal W,\varepsilon,\beta)$ be the least integer $n\geq 1$ for which every benchmark design with $d$ criteria admits an $(\varepsilon,\beta)$-valid mechanism using $n$ samples.
\end{definition}

For the full simplex, write
$\Npw(k,d,\varepsilon,\beta)\coloneqq\Npw(k,\Delt_d,\varepsilon,\beta)$ for simplicity.
The above definition is a weaker requirement than general adaptive estimation, which seeks an accurate population value for every adaptively chosen query or submitted model~\cite{dwork2015preserving}.  Instead, leaderboard validity asks only for an accurate value of the best model seen so far, per the relaxation introduced by Blum and Hardt~\cite{blum2015ladder}.  When $d=1$, our definition reduces exactly to the latter scalar leaderboard problem.

\subsection{Richer feedback interfaces}
\label{sec:richer-interfaces}
Beyond answering individual weighted requests, a benchmark may release criterion-level profiles or nondominated model sets, allowing downstream comparisons under many weightings. We formalize the validity of their best-so-far values through uniform approximation of the population lower envelope.

\begin{definition}[Frontier-value validity]
\label{def:frontier-validity}
Fix $(\mathcal Z,\mathcal F,\phi)$, $\mathcal W$, $n$, $k$, and $0<\varepsilon,\beta<1$. At round $t \in [k]$, the analyst submits $f_t$, and a mechanism with access to the hidden sample $S$ returns a nonempty
finite profiles set $\widehat{\mathcal P}_t\subset\mathbb R^d$.
The mechanism is $(\varepsilon,\beta)$-valid over $\mathcal W$ with $n$
 samples and $k$ submissions if, for every population law $P$ and every adaptive analyst
making at most $k$ submissions, when run on $S\sim P^n$, we have
\begin{equation*}
    \Prb\!\left[
    \max_{t\le k}\sup_{w\in\mathcal W}
    \left|
      \min_{a\in\widehat{\mathcal P}_t}\langle w,a\rangle
      -V_t^P(w)
    \right|
    \le\varepsilon
    \right]
    \ge1-\beta,
\end{equation*}
where the probability is over the sample $S$ and the mechanism and analyst.
\end{definition}

This definition targets the population value of the best submitted model
simultaneously for every allowed weighting.
It is easy to see that the above problem is no easier than weighted leaderboard validity, as we prove below.
Let $\Nfront(k,\mathcal W,\varepsilon,\beta)$ denote the least sample size for which a mechanism satisfying frontier-value validity exists for every benchmark design with $d$ criteria.

\begin{fact}[Frontier release is no easier]
\label{fact:frontier-no-easier}
For every $k,\mathcal W,\varepsilon,\beta$, we have
$\Nfront(k,\mathcal W,\varepsilon,\beta)
\ge \Npw(k,\mathcal W,\varepsilon,\beta)$.
\end{fact}
\begin{proof}
Given a valid frontier value release $\widehat{\mathcal P}_t$, answer a weighted request $w_t$ by
$
    \widehat v_t
    \mathrel{\coloneqq}\min_{a\in\widehat{\mathcal P}_t}\langle w_t,a\rangle.
$
Definition~\ref{def:frontier-validity} gives the desired bound at $w_t$.
\end{proof}
Thus every sample-complexity lower bound for weighted-leaderboard validity
proved below applies automatically to frontier-value release.
Conversely, Appendix~\ref{app:upper-proofs} shows that our upcoming Frontier Ladder mechanism
 attains frontier-value validity at the same order of
sample complexity, up to constants and logarithmic factors.  Section~\ref{sec:real-llm-experiment}
studies empirical Pareto frontier release as a richer feedback interface.

\section{Statistical Limits of Reuse}
\label{sec:statistical-limits}

Tracking the best-so-far value under one weighting can limit how much information a benchmark reveals.
With several weightings, different models can improve different rankings, creating more
opportunities for feedback to influence later submissions.  We now formalize this geometric
source of adaptivity.

\subsection{Frontier capacity}

For a nonempty compact convex set $C\subseteq\R^d$ and $b\in\R^d$, we define the \emph{exposure margin} as the largest margin by which some allowed weighting makes evaluation profile $b$ better
than every point in $C$:
\begin{equation}
    \gap_{\mathcal W}(b;C)
    \coloneqq\max_{w\in\mathcal W}
      \left\{
        \min_{c\in C}\langle w,c\rangle-\langle w,b\rangle
      \right\}.
    \label{eq:exposure-margin}
\end{equation}

\begin{definition}[Frontier capacity]
\label{def:frontier-capacity}
Let $K\subseteq [0,1]^d$ be nonempty, $\gamma>0$, and $\operatorname{conv}$ denote the convex hull.
\begin{enumerate}
\setlength{\itemsep}{0pt}
    \item A sequence $b_1,\ldots,b_m\in K, m\geq2,$ is
$\gamma$-convexly separated over $\mathcal W$ if
\begin{equation*}
    \gap_{\mathcal W}
    \left(b_t;\operatorname{conv}\{b_1,\ldots,b_{t-1}\}\right)
    \ge\gamma
    \qquad(t=2,\ldots,m).
\end{equation*}
\item The \emph{frontier capacity} of $K$ over $\mathcal W$ at margin $\gamma$,
denoted $\Mcap_{\mathcal W}(K,\gamma)$, is the supremum of the lengths of such sequences.
For evaluation profiles in the unit cube, write
$\Mcap_{\mathcal W}(d,\gamma)\coloneqq\Mcap_{\mathcal W}([0,1]^d,\gamma)$.
\end{enumerate}
\end{definition}

The margin $\gamma$ matters statistically, since an arbitrarily small improvement
cannot reliably reveal a sample fluctuation at fixed accuracy.
\begin{example}
 In two dimensions, let $\mathcal W=\Delt_2$ and consider evaluation profiles
 $b_1=(0.1,0.9), b_2=(0.4,0.4),b_3=(0.9,0.1).$
The weight vector $w_2=(1/2,1/2)$ gives scores $0.5$ and $0.4$ to $b_1$ and $b_2$, respectively, so $b_2$ improves on the previous frontier by $0.1$.  The weight vector $w_3=(0,1)$ gives $b_3$ score $0.1$, while every point in $\operatorname{conv}\{b_1,b_2\}$ has score at least $0.4$.  Thus $(b_1,b_2,b_3)$ is $0.1$-convexly separated over $\Delt_2$.
\end{example}

\begin{figure}[t]
    \centering
    \begin{tikzpicture}[
    scale=0.9,
    >=Latex,
    every node/.style={font=\small}
]

\coordinate (b1) at (4.0,1.9);
\coordinate (b2) at (5.0,2.1);
\coordinate (b3) at (5.6,3.1);
\coordinate (b4) at (5.0,4.1);
\coordinate (b5) at (3.7,4.4);
\coordinate (b6) at (3.0,3.4);

\fill[black!8]
    (b1)--(b2)--(b3)--(b4)--(b5)--(b6)--cycle;
\draw[thick]
    (b1)--(b2)--(b3)--(b4)--(b5)--(b6)--cycle;

\foreach \p in {b1,b2,b3,b4,b5,b6}
    \fill (\p) circle (1.5pt);

\node at (4.35,3.25) {$C_{t-1}$};

\coordinate (bt) at (1.55,1.25);
\fill (bt) circle (2.2pt);
\node[above=4pt] at (bt) {$b_t$};

\draw[thick]
    (0.85,1.95) -- (2.25,0.55);

\draw[thick,dashed]
    (1.20,4.70) -- (5.20,0.70);

\draw[->,thick]
    (0.72,0.52) -- (1.32,1.12)
    node[midway,left=2pt] {$w_t$};

\draw[<->]
    (2.10,0.70) -- (3.65,2.25)
    node[midway,above left] {$\gamma$};

\node[below] at (3.20,0.08) {
    $\displaystyle
    \langle w_t,b_t\rangle+\gamma
    \le
    \inf_{b\in C_{t-1}}\langle w_t,b\rangle$
};

\end{tikzpicture}
    \caption{\textbf{One step of frontier exposure.}
    The weight vector $w_t$ makes $b_t$ better than every point in the
    previous convex hull $C_{t-1}$ by margin $\gamma$.  Frontier capacity
    counts how often this can occur.  Later, Frontier Ladder (Alg.~\ref{alg:frontier-ladder}) uses a noisy version of this test
    to decide whether a new empirical profile improves its retained frontier.}
    \label{fig:convex_separation}
\end{figure}
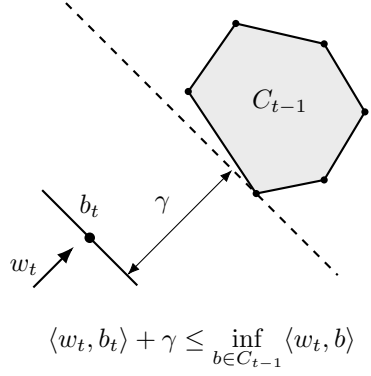
By compactness of $\mathcal W$ and linearity over the convex hull,
$b_1,\ldots,b_m$ is $\gamma$-convexly separated over $\mathcal W$ if and only if,
for every $t\ge2$, there is $w_t\in\mathcal W$, which we call a certificate weight for $b_t$, such that
\begin{equation}
    \langle w_t,b_s-b_t\rangle\ge\gamma
    \qquad(s<t).
    \label{eq:dual-certificate}
\end{equation}

Frontier capacity is a \emph{joint} geometric notion measuring how many profiles in $K$ can successively become best by margin
$\gamma$ under weightings in $\mathcal W$.
In contrast, classical geometric notions are typically not joint: packing would naturally use only profile geometry in our context, while covering would use only the weight family geometry.
Indeed, write $\mathsf P_\infty(K,\gamma)$ for the $\ell_\infty$ packing number of $K$,
the largest cardinality of a subset with pairwise $\ell_\infty$-distances at least $\gamma$,
and, for $r>0$, define the
$\ell_1$ covering number
of $\mathcal W$, the smallest cardinality of a subset $V\subseteq\mathcal W$ such that every $w\in\mathcal W$ lies within $\ell_1$-distance $r$ of some $v\in V$.
Then, we establish in Appendix~\ref{app:frontier-properties} the following upper estimate of frontier capacity:
\begin{equation}
    \Mcap_{\mathcal W}(K,\gamma)
    \le
    \min\left\{
        \mathsf P_\infty(K,\gamma),\;
        1+\left(1+\frac{2}{\gamma}\right)\mathcal N_1(\mathcal W,\gamma/2)
    \right\}.
    \label{eq:sep-packing-upper}
\end{equation}
Even the minimum expression above can exceed frontier capacity by a factor exponential in $d$
(Remark~\ref{rem:simultaneous-looseness}).
This suggests that frontier capacity is not trivially subsmused by known packing and convering metrics. We further discuss the relation with a convex-geometry notion called convexified packing in Remark~\ref{rem:convexified-packing}.

A final property that will be useful for our main results is that frontier capacity over the full simplex can be exponentially large, even at a constant margin.  The following geometric bound tracks the tradeoff
between the number of exposed profiles and their margin.  Its proof in
Appendix~\ref{app:simplex-construction} uses a construction mapping a $q$-ary error-correcting code to a
convex quadratic graph.

\begin{restatable}[Simplex frontier capacity]{lemma}{qarycode}
\label{lem:qary-code}
There is an absolute constant $c_0>0$ such that, for all integers $d,q\ge2$,
\begin{equation}
    \Mcap_{\Delta_d}\!\left(
        d,\frac{1}{48(q-1)^2}
    \right)
    \ge
    \left\lfloor q^{c_0(d-1)}\right\rfloor .
\end{equation}
\end{restatable}

\subsection{Lower bounds on sample complexity}
\label{sec:capacity-reduction}

We now show that sufficiently large frontier capacity makes tracking
best values as statistically hard as general adaptive estimation.
We prove this by embedding adaptive queries into well-separated evaluation profiles whose
margins keep each new submission best under its certificate weighting,
so its leaderboard value reveals the adaptive query mean at every round.

\begin{theorem}[Lower bound from frontier capacity]
\label{thm:capacity-lower}
There are universal constants $c,C,\varepsilon_0,\beta_0>0$ such that,
for every integer $d\ge1$, every nonempty compact
$\mathcal W\subseteq\Delta_d$, every integer $k\ge4$,
$0<\varepsilon\le\varepsilon_0$, and $0<\beta\le\beta_0$,
\begin{equation}
    \Npw(k,\mathcal W,\varepsilon,\beta)
    \ge
    c\left[
        \min\left\{\sqrt{k},\sqrt{\Mcap_{\mathcal W}(d,C\varepsilon)}\right\}+
        \frac{\log k}{\varepsilon^2}
    \right].
    \label{eq:capacity-lower}
\end{equation}
The lower bound holds with weightings fixed before the interaction and with binary classifiers under the $0$--$1$ loss.
\end{theorem}

\subsubsection{Proof of Theorem~\ref{thm:capacity-lower}}
We first show that weighted leaderboards are at least as hard as adaptive
query estimation: any valid weighted leaderboard mechanism yields an accurate estimator of
adaptive query means using the same sample.  Frontier capacity determines
how many queries the reduction can embed, while the exposure margin
determines the accuracy of the recovered answers.

\begin{lemma}[Reduction]
\label{lem:convex-reduction}
Let $\Naq(m,\alpha,\beta)$ denote the least sample size needed to estimate
the population means $\E_P[g_t(Z)]$ of $m$ adaptively chosen queries
$g_t:\mathcal Z\to[-1,1]$ within error $\alpha$, simultaneously with failure
probability at most $\beta$.  Each query may depend on earlier answers,
and all answers use the same hidden i.i.d.\ sample.

For every integer $k\ge1$, $0<\gamma\le1$, and $0<\varepsilon,\beta<1$, we have
\begin{equation}
    \Npw(k,\mathcal W,\varepsilon,\beta)
    \ge
    \Naq\!\Big(\!
        \min\{k,\Mcap_{\mathcal W}(d,\gamma)\},
        \frac{4\varepsilon}{\gamma},\beta
    \Big).
    \label{eq:convex-reduction}
\end{equation}
\end{lemma}
\begin{proof}[Proof]
Put $m_\gamma\coloneqq\min\{k,\Mcap_{\mathcal W}(d,\gamma)\}$ and suppose an
$(\varepsilon,\beta)$-valid weighted leaderboard uses $n$ samples.
The packing bound~\eqref{eq:sep-packing-upper} makes the capacity a finite
integer, so we may choose a $\gamma$-convexly separated sequence
$b_1,\ldots,b_{m_\gamma}\in[0,1]^d$ with certificates
$w_2,\ldots,w_{m_\gamma}\in\mathcal W$ satisfying
\eqref{eq:dual-certificate}; choose any $w_1\in\mathcal W$.

Set $\rho\coloneqq\gamma/4$, and let $\mathcal G$ be the class of functions
$g:\mathcal Z\to[-1,1]$.  Use the model class
$\mathcal F=[0,1]^d\times\mathcal G$ and evaluation map
\begin{equation*}
    \phi((b,g),z)
    \coloneqq
    (1-2\rho)b+\rho(1+g(z))\one\in[0,1]^d,
\end{equation*}
where $\one$ is the vector of ones.  The contraction of $b$ leaves room
for the perturbation while keeping every evaluation in the unit cube.
For the analyst's $t$-th adaptive query $g_t$, submit
$f_t\coloneqq(b_t,g_t)$ with weight $w_t$.  Writing
$\mu_t\coloneqq\E_P[g_t(Z)]$, its population profile is
\begin{equation*}
    R_P(f_t)=(1-2\rho)b_t+\rho(1+\mu_t)\one.
\end{equation*}
Since $\langle w_t,\one\rangle=1$, the perturbation passes unchanged into
the weighted score.  The exposure margin then absorbs the difference
between any two query means: for every $s<t$,
\begin{align*}
    \langle w_t,R_P(f_s)-R_P(f_t)\rangle
    =(1-2\rho)\langle w_t,b_s-b_t\rangle+\rho(\mu_s-\mu_t)
    \ge(1-2\rho)\gamma-2\rho
      =\frac{\gamma(1-\gamma)}2\ge0.
\end{align*}
Thus $f_t$ attains the minimum under $w_t$ (a tie is harmless since we care about the value), giving
\begin{equation*}
    V_t^P(w_t)
    =(1-2\rho)\langle w_t,b_t\rangle+\rho(1+\mu_t).
\end{equation*}
Subtracting the known baseline and rescaling the leaderboard answer,
return
\begin{equation*}
    y_t\coloneqq
    \frac{\widehat v_t-(1-2\rho)\langle w_t,b_t\rangle-\rho}{\rho}.
\end{equation*}
On the weighted leaderboard validity event, simultaneously for every $t\le m_\gamma$, we have
\begin{equation*}
    |y_t-\mu_t|
    =\frac{|\widehat v_t-V_t^P(w_t)|}{\rho}
    \le\frac{4\varepsilon}{\gamma}.
\end{equation*}
Each $y_t$ depends only on the leaderboard transcript, so the analyst can
use it to choose its next query.  We have therefore answered $m_\gamma$
adaptive queries using the same $n$ observations, proving
\eqref{eq:convex-reduction}.
\end{proof}
In the proof above, the evaluation profiles and certificate weights are fixed before the interaction.
Also, Proposition~\ref{prop:classification-realization}
realizes the construction by binary classifiers without increasing the
sample size.

\emph{Adaptive cost.}
Let $c_3,c_4,c_5$ be the constants in Lemma~\ref{fact:aq-lower}, which states
that $\Naq(m,c_3,c_4)\ge c_5\sqrt m$ for every $m\ge1$.
Choose $C=4/c_3$ and $\gamma=C\varepsilon$, with
$\varepsilon_0\le1/C$ and $\beta_0\le c_4$.
Then $0<\gamma\le1$ and the recovered query accuracy is $c_3$.
The lemma and the constant-accuracy adaptive-query lower bound give
\begin{equation*}
    \Npw(k,\mathcal W,\varepsilon,\beta)
    \ge c_5\min\left\{\sqrt{k},\sqrt{\Mcap_{\mathcal W}(d,C\varepsilon)}\right\}.
\end{equation*}

\emph{Nonadaptive cost.}
Even models fixed before the sample is drawn incur a selection cost under
one weighting.  Proposition~\ref{thm:log-lower}
embeds the construction of Blum--Hardt~\cite{blum2015ladder} along a fixed weight in $\mathcal W$:
\begin{equation}
    \Npw(k,\mathcal W,\varepsilon,\beta)
    \ge\frac{\log k}{256\varepsilon^2}
    \label{eq:selection-lower}
\end{equation}
for $k\ge4$, $\varepsilon\le1/16$, and $\beta\le1/4$.
Adding the two lower bounds and dividing by two proves the theorem with
$\varepsilon_0\le\min\{1/C,1/16\}$,
$\beta_0\le\min\{c_4,1/4\}$, and $c\le\tfrac12\min\{c_5,1/256\}$.

\subsubsection{Winner-only feedback}
Surprisingly, this adaptive cost persists when the leaderboard reveals only an
approximately best model's identity.  
To show this, we augment our reduction strategy with binary search by using two submissions with the same exposed
baseline profile: one carries an unknown query mean, and the other a well-chosen
threshold.  The exposure margin excludes earlier submissions from winning,
so the returned identity compares the mean with the threshold, up to the
selection error.  Repeating this comparison at fresh exposed profiles
recovers the mean by binary search.
We formalize this result below.

Recall $V_t^P(w)=\min_{s\le t}\langle w,R_P(f_s)\rangle$.
At round $t$, the mechanism returns an index $\widehat s_t\in[t]$ satisfying
\begin{equation*}
    \langle w_t,R_P(f_{\widehat s_t})\rangle
    \le V_t^P(w_t)+\varepsilon
\end{equation*}
simultaneously over all rounds with probability at least $1-\beta$.
Let $\Nsel(k,\mathcal W,\varepsilon,\beta)$ denote the corresponding
sample complexity.
The proof of the lower bound below is in Appendix~\ref{app:winner-only}. 
\begin{restatable}[Winner-only lower bound]{theorem}{selectionlower}
\label{thm:selection-capacity-lower}
There are universal constants $c,C,\varepsilon_0,\beta_0>0$ such that,
for every integer $d\ge1$, every nonempty compact
$\mathcal W\subseteq\Delta_d$, every integer $k\ge2$,
$0<\varepsilon\le\varepsilon_0$, and $0<\beta\le\beta_0$,
\begin{equation}
    \Nsel(k,\mathcal W,\varepsilon,\beta)
    \ge c\min\left\{\sqrt{k},\sqrt{\Mcap_{\mathcal W}(d,C\varepsilon)}\right\}.
    \label{eq:selection-capacity-lower}
\end{equation}
The lower bound holds with weightings fixed before the interaction and with
binary classifiers under the $0$--$1$ loss.
\end{restatable}

\subsection{The Frontier Ladder mechanism}
\label{sec:frontier-ladder}

Our adaptive lower bound's strength rests upon frontier capacity.  For the upper bound, Frontier Ladder
uses the same capacity to control its number of updates, limiting how
often it releases new information.

Recall the Ladder principle~\cite{blum2015ladder}: a new score is released only when the
submitted model improves on the best value seen so far by a margin.
With changing weightings, however, a model that makes little progress under the current
weighting may become best under a later one.  A multi-criteria analogue therefore
cannot decide whether to discard a profile using only the weighting requested when it arrives.

\begin{example}[Current-weight updates miss winners]
\label{ex:current-weight-failure}
Suppose $b_1=(0,1)$ is retained when $b_2=(1,0)$ arrives.
Under weighting $(1,0)$, $b_2$ offers no improvement and fails
any positive update threshold.  Under weighting $(0,1)$, however,
$b_2$ is better by $1$.  Thus discarding profiles based only on
the current weighting can miss a later improvement.
\end{example}

We thus test whether a new profile improves
the retained lower envelope under \emph{any} weighting in $\mathcal W$.

\paragraph{The mechanism.}
Frontier Ladder represents its retained lower envelope by a finite set $A_t$.
For a finite $A\subset\mathbb R^d$, write
\begin{equation*}
    F_A(w)\coloneqq\min_{a\in A}\langle w,a\rangle.
\end{equation*}
A new profile is retained only if it improves this envelope by a
margin under some $w\in\mathcal W$.  To limit the hidden-sample leakage from these decisions, we use the
\emph{sparse-vector} technique~\cite{dwork2009complexity,dwork2014foundations}, which adds noise to threshold comparisons and
whose privacy cost depends mainly on the number of positive reports,
which frontier capacity bounds.  For sensitivity-$1/n$ queries
$q_t$ on the hidden sample, threshold $\tau$, and a cap $L$ on positive reports,
the primitive in Lemma~\ref{fact:sv} guarantees, with tolerance $u$, that
\begin{equation*}
    \text{negative report}\Longrightarrow q_t(S)\le\tau+u,
    \qquad
    \text{positive report}\Longrightarrow q_t(S)\ge\tau-u.
\end{equation*}
Algorithm~\ref{alg:frontier-ladder} applies it to the largest improvement over
the retained envelope, and adds Gaussian noise to each retained profile.
This is extends the principle used by Shaky Ladder to high dimension~\cite{hardt2017climbing}.

\begin{algorithm}[t]
\caption{Frontier Ladder}
\label{alg:frontier-ladder}
\small
\setlength{\abovedisplayskip}{3pt}
\setlength{\belowdisplayskip}{3pt}
\setlength{\abovedisplayshortskip}{0pt}
\setlength{\belowdisplayshortskip}{3pt}
\begin{algorithmic}[1]
\Require Hidden sample $S=(Z_1,\ldots,Z_n)$, weight family $\mathcal W$,
         horizon $k$, target $(\varepsilon,\beta)$
\State Set $u\gets\frac{\varepsilon}{80},
    \,
    \tau\gets4u,
    \,
    L\gets
    1+\min\left\{
        k,\,
        \Mcap_{\mathcal W}\!\left(d,\frac{\varepsilon}{40}\right)
    \right\}$.
\State Set $\eta_0\gets\tfrac{\varepsilon}{256},
    \,
    \delta_0\gets\tfrac{\varepsilon\beta}{256},
    \,
    \zeta_0\gets\tfrac{\varepsilon\beta}{128}$,
and $\sigma\gets
    \tfrac{256c_1}{n\varepsilon}
      \sqrt{dL\log(256c_1/(\varepsilon\beta))}$,
with $c_1$ from Lemma~\ref{lem:gaussian}.
\State Initialize $A_0\gets\{\one\}$ and sparse vector with threshold $\tau$,
cap $L$, and parameters $(\eta_0,\delta_0,\zeta_0)$ (Lemma~\ref{fact:sv}).
\For{$t=1,\ldots,k$}
    \State Receive $(f_t,w_t)$ and compute $\widehat R_t(S)
        \gets
        \frac1n\sum_{i=1}^n\phi(f_t,Z_i)$.
    \State Compute and submit $S\mapsto\Gamma_t(S)$ to sparse vector:
    \begin{equation}
        \Gamma_t(S)
        \gets
        \max_{w\in\mathcal W}
        \left\{
            F_{A_{t-1}}(w)
            -
            \langle w,\widehat R_t(S)\rangle
        \right\}.
        \label{eq:frontier-novelty}
    \end{equation}
    \If{the report is negative} $A_t\gets A_{t-1}$.
    \Else
        \State Draw $\xi_t\sim\mathcal N(0,\sigma^2I_d)$  and set $\widetilde R_t
            \gets
            \operatorname{clip}_{[0,1]^d}
            \bigl(\widehat R_t(S)+\xi_t\bigr),
            \,
            A_t\gets A_{t-1}\cup\{\widetilde R_t\}$.
    \EndIf
    \State Output $\widehat v_t\gets F_{A_t}(w_t)$.
\EndFor
\end{algorithmic}
\end{algorithm}

Write $\mathsf L_{\mathcal W}(k,\varepsilon)\coloneqq
1+\min\{k,\Mcap_{\mathcal W}(d,\varepsilon/40)\}$ for the cap on positive reports
in Algorithm~\ref{alg:frontier-ladder}.
Conditional on the preceding transcript, the query $\Gamma_t$ defined in
\eqref{eq:frontier-novelty} is $1/n$-sensitive,
so it is a valid sparse-vector query.  If sparse vector reaches its cap, the
mechanism returns a fixed value on subsequent rounds; the proof shows that this
does not occur on the utility event.
Aside, note that $\Gamma_t$ can be computed efficiently via linear programming when operating over the full simplex $\mathcal{W}=\Delta_d$.

\begin{restatable}[Frontier Ladder validity]
{theorem}{geometricupper}
\label{thm:geometric-upper}
Let $L\coloneqq\mathsf{L}_{\mathcal W}(k,\varepsilon)$. There is a universal constant
$C\ge 1$ such that, for every nonempty compact
$\mathcal W\subseteq\Delta_d$, every $k\ge 2$, and
$0<\varepsilon,\beta\le 1/10$, Frontier Ladder is valid whenever
\begin{equation}
n
\ge
\frac{C}{\varepsilon^2}
\left[
\sqrt{
    L
    \log\bigl(C/(\varepsilon\beta)\bigr)
    \log\bigl(CkdL/(\varepsilon\beta)\bigr)
    \left(
        d+\log\bigl(CkdL/(\varepsilon\beta)\bigr)
    \right)
}
+
\log(k)
\right].
\label{eq:geom-upper}
\end{equation}
Consequently, the right-hand side upper-bounds
$\Npw(k,\mathcal W,\varepsilon,\beta)$.
\end{restatable}

The packing bound in~\eqref{eq:sep-packing-upper} gives
$\mathsf{L}_{\mathcal W}(k,\varepsilon)
\le1+\min\{k,(1+\lceil40/\varepsilon\rceil)^d\}$.
For restricted weight families, frontier capacity may be substantially smaller, as we show in Proposition~\ref{prop:frontier-properties} for sparse weightings.

Noise makes the threshold decisions and released profiles stable with respect
to the hidden sample, allowing their accuracy to transfer to the population.
A deterministic variant with rounded retained profiles also admits a transcript-counting
argument, as in the original Ladder~\cite{blum2015ladder}; we prefer randomization here because it yields
the square-root dependence on the number of updates in \eqref{eq:geom-upper}.

The retained set in Frontier Ladder also contains more information than the weighted answer alone:
Proposition~\ref{prop:frontier-release-upper} in Appendix~\ref{app:upper-proofs}
shows that publishing $A_t$ satisfies frontier-value validity
(Definition~\ref{def:frontier-validity}) at the same order of sample complexity.

\paragraph{A dimension-free fallback.}
When frontier capacity is large, it is better to ignore the geometry and release the
best empirical value directly using a standard adaptive-query mechanism.  Indeed, conditional
on the preceding transcript, the empirical best-so-far value at round $t$ is
\begin{equation*}
    Q_t(S)
    =
    \min_{s\le t}
    \frac1n\sum_{i=1}^n
    \langle w_t,\phi(f_s,Z_i)\rangle
\end{equation*}
and changes by at most $1/n$ if a single observation in the hidden sample is
replaced.  We can therefore apply a private adaptive-query mechanism and
transfer its sample-accuracy to the population.

\begin{restatable}[Dimension-free fallback validity]{theorem}{genericupper}
\label{thm:generic-upper}
Fix $0<\varepsilon_0,\beta_0\le1/10$.
There is a constant $C\ge1$ such that,
for every nonempty compact $\mathcal W\subseteq\Delta_d$ and every $k\ge2$,
\begin{equation}
    \Npw(k,\mathcal W,\varepsilon_0,\beta_0)
    \le
    C\sqrt{k}.
    \label{eq:generic-upper}
\end{equation}
\end{restatable}
We obtain the $\mathcal{O}(\sqrt{k})$ rate by combining the transfer theorem of
Bassily et al.~\cite{bassily2016stability} with Dagan--Kur's bounded-noise
mechanism (Lemma~\ref{fact:bounded-noise})~\cite{dagan2022bounded}, which uses a specially chosen noise
density with bounded support.
This removes the extra $\sqrt{\log k}$ factor incurred by controlling the largest
of $k$ Gaussian perturbations under the same transfer argument.
Appendix~\ref{app:upper-proofs} gives the proof and, for completeness, the Gaussian bound for
general accuracy and confidence targets.

\subsection{Phase diagram}

We now specialize the statistical bounds to the full simplex
$\mathcal W=\Delta_d$, i.e., users may combine evaluation criteria arbitrarily.  We fix 
accuracy and failure probability to make the dependence on $k$ and $d$ clearer.

\begin{corollary}
\label{thm:phase-diagram}
There are universal constants
$\varepsilon_0,\beta_0,c,C>0$ such that, for every $k\ge4$ and
$d\ge2$, we have
\begin{empheq}[left=\empheqlbrace]{align}
    \Npw(k,d,\varepsilon_0,\beta_0)
    &\ge
    c\left[
      \log k+
      \min\{\sqrt{k},\exp(cd)\}
    \right],\\
    \Npw(k,d,\varepsilon_0,\beta_0)
    &\le
    C\min\left\{
      \exp(Cd)(\log k+d),\,
      \sqrt{k}
    \right\}.
\end{empheq}
\end{corollary}
\begin{proof}
For the lower bound, Lemma~\ref{lem:qary-code} with $q=2$ gives
$\Mcap_{\Delta_d}(d,C\varepsilon_0)\ge\lfloor2^{c_0(d-1)}\rfloor$.
Plugging into Theorem~\ref{thm:capacity-lower} proves the lower bound
after adjusting universal constants.
For the upper bound, \eqref{eq:generic-upper} gives the second term.
For Frontier Ladder, the packing bound in Proposition~\ref{prop:frontier-properties} gives
$\mathsf{L}_{\Delta_d}(k,\varepsilon_0)\le 1+C^d$ for an absolute $C>1$;
plugging this into Theorem~\ref{thm:geometric-upper} and absorbing
polynomial factors in $d$ into the exponential gives the first term.
Taking the better upper bound proves the result.
\end{proof}

For $d$ at least a sufficiently large constant times $\log k$, these bounds
give $\Theta(\sqrt{k})$, matching general adaptive
estimation at fixed accuracy~\cite{steinke2015interactive}.
Every fixed dimension instead gives $\Theta_d(\log k)$ dependence on the
number of submissions.  Between these regimes, the adaptive cost is exponential
in $d$, up to constants in the exponent and logarithmic factors, until it reaches
the $\sqrt{k}$ scale.  Appendix~\ref{app:variable-accuracy} discusses the
variable-accuracy bounds.
Combining Theorem~\ref{thm:selection-capacity-lower} with
Lemma~\ref{lem:qary-code} likewise gives the high-dimensional
$\sqrt{k}$ lower bound for winner-only feedback at sufficiently small
fixed accuracy and failure probability.

For general weight families, the bounds reveal a common square-root
dependence on frontier capacity.  At the fixed accuracy and failure probability
above, combining
Theorems~\ref{thm:capacity-lower}, \ref{thm:geometric-upper}, and
\ref{thm:generic-upper} gives, for $k\ge4$ and any nonempty compact
$\mathcal W\subseteq\Delta_d$, that
\begin{empheq}[left=\empheqlbrace]{align}
    \Npw(k,\mathcal W,\varepsilon_0,\beta_0)
    &\ge c\left[\log k+
        \sqrt{\min\{k,\Mcap_{\mathcal W}(d,C_0\varepsilon_0)\}}\right],
        \label{eq:capacity-comparison}\\
    \Npw(k,\mathcal W,\varepsilon_0,\beta_0)
    &\le C\min\left\{
        \log k+\sqrt{\Mcap_{\mathcal W}(d,\varepsilon_0/40)
            \log(kd)\bigl(d+\log(kd)\bigr)},\,
            \sqrt{k}
            \right\},
        \label{eq:capacity-comparison-upper}
\end{empheq}
where $c,C,C_0>0$ are universal constants.
The remaining gaps are the different capacity margin scales and the factor
$\sqrt{\log(kd)\bigl(d+\log(kd)\bigr)}$ in the upper bound;
its dimension dependence comes from privatizing full evaluation profiles.
Whether weighted answers alone can avoid this cost remains open.

Restricting the supported weightings can make both capacities substantially
smaller than for the full simplex.
For the weight family $\mathcal W=\{e_1,\ldots,e_d\}$, corresponding
to comparisons on one criterion per round,
$\Mcap_{\mathcal W}(d,\gamma)=1+d\lfloor1/\gamma\rfloor$ for
$0<\gamma\le1$: each coordinate permits at most $\lfloor1/\gamma\rfloor$
record decreases of at least $\gamma$, and lowering coordinates one at a
time attains the bound.
Thus, at fixed accuracy, the adaptive lower bound for these weightings
reaches the $\sqrt{k}$ scale only when $d$ is of order $k$, compared with
order $\log k$ when arbitrary convex combinations are allowed.

\section{Experiments on Large Language Model Multi-Task Benchmarks}
\label{sec:experiments}
\label{sec:real-llm-experiment}

We compare scalar and multi-criteria feedback in controlled attacks on four LLM benchmark snapshots, measuring reused-to-held-out score gaps and false-winner rates. A false winner is a final model that ranks first on the reused split but not on the held-out split. Inspired by the boosting attack~\cite{blum2015ladder}, we randomly construct many prompt routers between two well-chosen LLMs, then combine the released feedback into one final router via majority voting. This gives a cheap way to vary evaluation profiles, since every router can be scored from archived item scores without model inference.
Appendix~\ref{app:controlled-supplements} complements these experiments with weighted leaderboard simulations, where validity error grows as $d$ permits more separated profiles until the submission budget becomes the limiting factor.\footnote{Code and benchmark data: \url{https://github.com/ysfalh/multi-criteria-benchmarks}.}

\begin{figure}[!t]
    \centering
    \includegraphics[width=\linewidth]{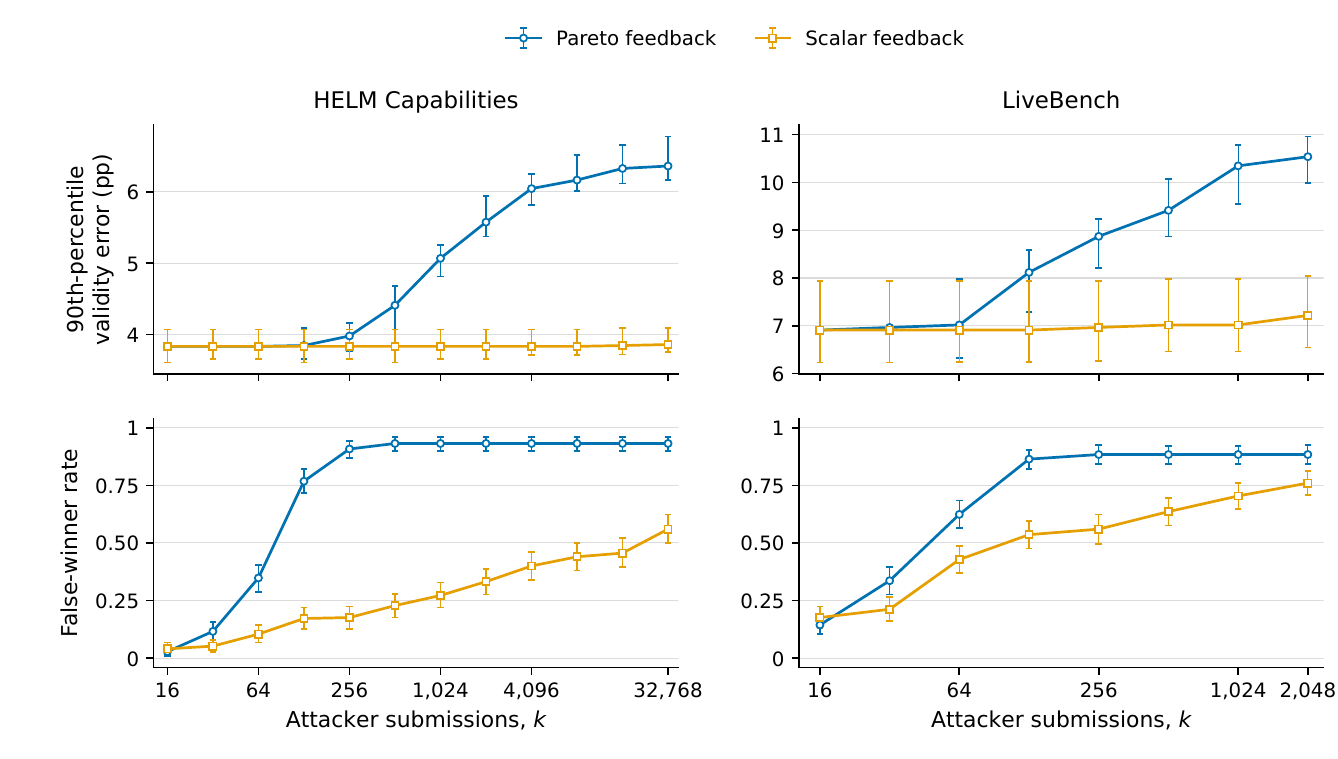}
    \caption{\textbf{Score gaps and false-winner rates across attack budgets.}
    The attacker observes only new nondominated task profiles (\emph{Pareto} feedback)
    or new aggregate records (\emph{Scalar} feedback).
    \emph{Top:} $90$th-percentile leaderboard gaps, comparing
    the published aggregate loss on reused $S$ and the best-so-far aggregate loss on held-out $T$.
    The final router's  loss gaps are in
    Figure~\ref{fig:final-router-gap}.
    \emph{Bottom:} the fraction of trials in which the final router ranks first
    against genuine models on $S$ but not on $T$; ties share first place.
    Bars show  $95\%$ bootstrap intervals over $250$ paired trials.}
    \label{fig:real-budget-results}
\end{figure}

\paragraph{Attack construction.}
Choose two endpoint models $A$ and $B$. For each candidate $t$, a random
binary rule $g_t(x)\in\{-1,+1\}$, fixed before the reused benchmark is
observed, determines which endpoint answers prompt $x$. The resulting router
$R_t$ has loss
\begin{equation}
    \ell_{R_t}(x)
    =
    \frac{\ell_A(x)+\ell_B(x)}{2}
    +
    \frac{g_t(x)}{2}
    \bigl(\ell_A(x)-\ell_B(x)\bigr).
    \label{eq:llm-router}
\end{equation}
Thus routing changes the loss only where the endpoints receive different scores, and different random rules produce different criterion-level profiles.

We submit the same routers in the same order under two restricted feedback interfaces, chosen to align with the premise that leaderboards only need to track the ``best'' model.
\emph{Scalar feedback} releases a new aggregate value only when a router improves
on the best earlier aggregate loss. \emph{Pareto feedback} releases a router's
criterion-level profile unless an earlier submission has no larger loss on every
criterion and strictly smaller loss on at least one. The Pareto analyst retains
released rules whose aggregate loss beats the midpoint of $A$ and $B$, provided
both endpoint profiles have been released; the scalar analyst retains aggregate
record setters. After all candidates, the final router selects, on each prompt,
the endpoint chosen by a majority of the retained rules. This is the only
adaptive submission.

\paragraph{Benchmark setup.}
We use four snapshots with public per-item evaluation records:
(i)~HELM Capabilities (v1.15.0, 2025)~\cite{liang2023holistic}, for HELM's
standardized transparent evaluation on five challenging scenarios;
(ii)~LiveBench (October 2024)~\cite{livebench}, designed to limit contamination
through regularly refreshed questions;
(iii)~HELM Lite (v1.13.0, 2025)~\cite{liang2023holistic}, which applies the
HELM framework to ten scenarios spanning question answering,
reasoning, and translation; and
(iv)~Hugging Face's Open LLM Leaderboard (v1, 2024)~\cite{open-llm-leaderboard},
focusing on open-weight models.

Within each task, we partition examples into a calibration split $C$ used
only to choose $A,B$, a reused split $S$, and a held-out split $T$ that is only
consulted after the interaction. We choose Pareto-nondominated endpoints with
complementary errors on $C$, so routing can produce meaningful criterion
tradeoffs. Table~\ref{tab:calibration-pairs} gives the pairs and selection
diagnostics. Appendix~\ref{sec:llm-supplement} specifies the model pools and
split sizes.

\paragraph{Evaluation.}
For feedback type $c\in\{\mathrm{Pareto},\mathrm{scalar}\}$, let
$F_T^c(t)$ be the best aggregate loss among the first $t$ submissions on $T$,
with equal weight on the benchmark's categories. Let $\widehat F_S^c(t)$ be
the best aggregate loss on $S$ which has been
released by round $t$. We report the published-value error
\begin{equation*}
    E_k^c
    \coloneqq
    \max_{t\le m+k}
    \left|\widehat F_S^c(t)-F_T^c(t)\right|,
\end{equation*}
where $m$ genuine models precede the $k$ attack submissions.
We approximate
the population in Def.~\ref{def:weighted-validity} with the held-out
split and evaluate only the equal-category weighting; the worst error
over all weightings can only be larger.
Table~\ref{tab:llm-summary} reports the $90$th percentile of $E_k^c$ and the
\emph{false-winner rate}. The latter is defined as, for the final router $G_i^c$ at budget $k$ in
trial $i$ and with rankings taken against original models:
\begin{equation*}
    p_k^c
    \coloneqq \frac{1}{N_{\text{trials}}}\sum_{i=1}^{N_{\text{trials}}}
    \mathbf{1}\{\operatorname{rank}_{S_i}(G_i^c)=1,\;
                 \operatorname{rank}_{T_i}(G_i^c)>1\},
    \qquad N_{\text{trials}}=250.
\end{equation*}

\begin{figure}[!t]
    \centering
    \includegraphics[width=\linewidth]{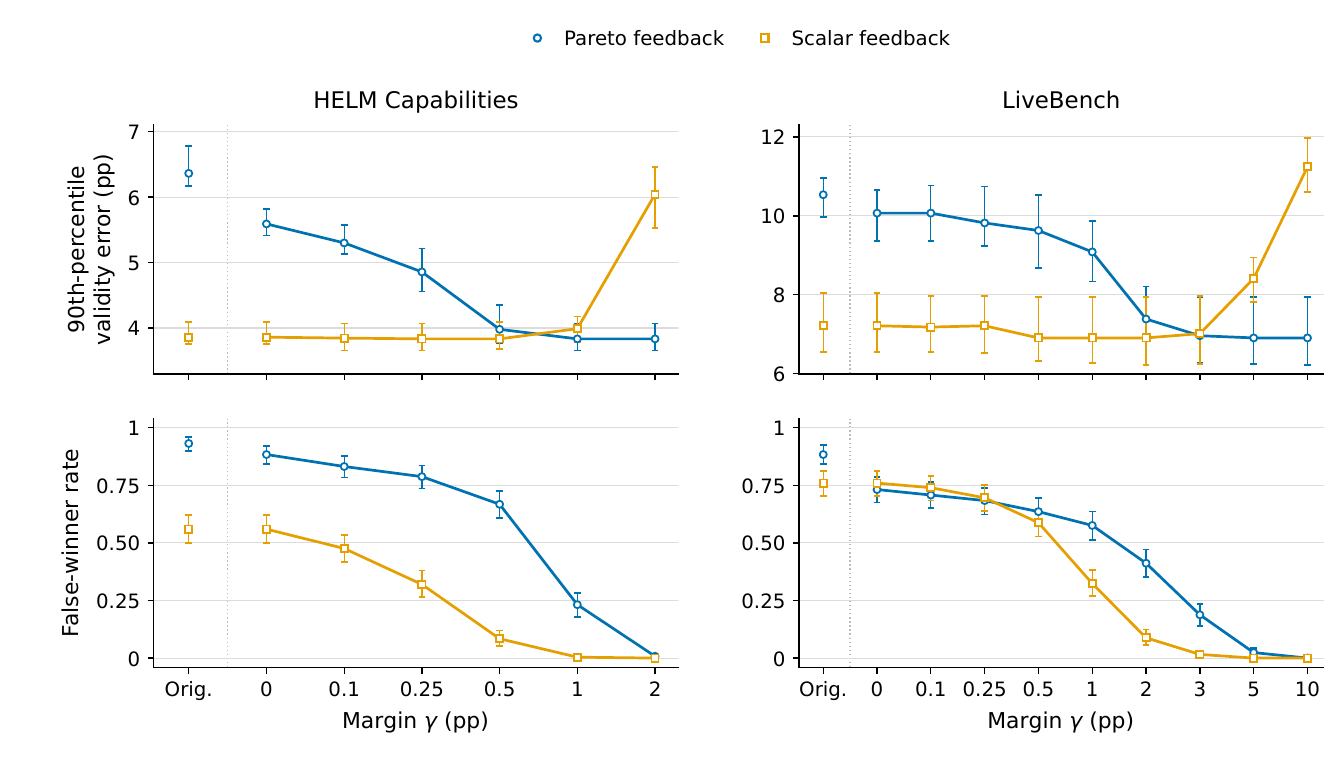}
    \caption{\textbf{Restricting feedback through the geometry of frontier capacity.}
    We further restrict the interfaces in Figure~\ref{fig:real-budget-results}
    by requiring improvement greater than $\gamma$ over every earlier
    submission under some criterion weighting (Pareto feedback) or
    the fixed equal-category weighting (scalar feedback).
    \emph{Top:} $90$th-percentile published-value error $E_k^c$.
    \emph{Bottom:} false-winner rates: first against genuine models on $S$
    but not on $T$, counting ties as first place.
    Bars show $95\%$ bootstrap intervals over $250$ paired trials.
    Attack budgets are $32{,}768$ for HELM Capabilities and $2{,}048$ for
    LiveBench. ``Orig.'' denotes the interfaces without an added margin.}
    \label{fig:llm-feedback}
\end{figure}

\noindent
\textbf{Results.}
In Table~\ref{tab:llm-summary}, Pareto feedback yields a larger
$90$th-percentile leaderboard score gap than scalar feedback on each of the four
benchmarks. False-winner rates range from $88.4\%$ to $100\%$ under Pareto
feedback, compared with $0.8\%$ to $76\%$ under scalar feedback. Averaged
equally across benchmarks, the rates are $95.2\%$ and $48.5\%$, respectively.
No final router ranks first on $T$ in any of the trials.

Figure~\ref{fig:real-budget-results} shows how these effects vary with the
number of submissions. Under Pareto feedback, false-winner rates rise rapidly
and then level off, while the leaderboard score gap continues to grow. The scalar
interface produces fewer false winners at the headline budgets.
Appendix~\ref{sec:llm-supplement} reports
confidence intervals, full rank results, and feedback counts.

\paragraph{Intervening on frontier geometry.}
The original interfaces already restrict feedback to new nondominated
profiles or aggregate records. Inspired by the geometry underlying frontier capacity, Figure~\ref{fig:llm-feedback} further requires
the empirical exposure margin from Eq.~\eqref{eq:exposure-margin} to exceed $\gamma$.
Even with this restriction, Pareto feedback can still permit more spurious progress than scalar feedback:
on HELM Capabilities, at $\gamma=0.5$ percentage points, the false-winner
rate is $66.8\%$ under Pareto feedback and $8.4\%$ under scalar feedback.
Increasing $\gamma$ reduces false-winner rates over the range in
Figure~\ref{fig:llm-feedback}. One limitation of this intervention is that larger margins can also withhold genuine
improvements and increase validity error. Figure~\ref{fig:margin-gates-extended}
extends the range and shows this upturn for Pareto feedback as well.

\section{Discussion \& Open Problems}
\label{sec:discussion}

How reusable, then, are benchmarks with rich feedback? Our work suggests that a benchmark may become harder to reuse as it becomes useful for more decisions. This tension arises because a single aggregate gives developers one score to improve, while users may care about different  criteria. Accommodating these differences introduces additional rankings whose leaders can  guide  development.

Our theory shows that this broader role can substantially increase the data needed for reliable reuse, even when feedback reveals only the identity of an approximately best model and the criteria weightings are fixed in advance. Although ImageNet relative improvements largely persisted on newly collected test data~\cite{recht2019imagenet}, this evidence does not establish reliable reuse when development draws on many task-specific comparisons. Restricting the supported comparisons may reduce the data required, but at the cost of leaving some users unable to compare models on the tasks that matter to them.

How much of this worst-case difficulty arises in ordinary model development? Answering this seems nontrivial and requires tracing which comparisons guide later models and checking whether the resulting gains persist on fresh samples from the  population. If practice remains reliable despite rich feedback, existing (single-criterion) explanations based on model similarity~\cite{mania2019similarity}, structured use of feedback~\cite{zrnic2019natural}, and problem structure~\cite{feldman2019advantages} offer starting points. Can these effects be measured across the relevant task weightings and incorporated into a theory that explains the gap from the worst case? Where rich feedback does produce spurious progress, can mechanisms simpler and more computationally efficient than Frontier Ladder preserve useful comparisons at realistic sample sizes?

Finally, several compelling technical questions remain open. First, what is the scale of frontier capacity $\Mcap_{\mathcal W}(d,\gamma)$ over natural weight families as the margin $\gamma$ shrinks? Even for the simplex $\mathcal W=\Delta_d$, the current upper and lower bounds have different dependence on $\gamma$. Tightening these could lead to tighter variable-accuracy rates. Second, does frontier capacity solely characterize the adaptive sample complexity? Frontier Ladder pays an additional cost for stabilizing full $d$-dimensional profiles, and it is open whether weighted-leaderboard validity can instead be achieved at about $\Theta({\min\{\sqrt k,\sqrt{\Mcap_{\mathcal W}(d,\gamma)}\}})$ adaptive cost.
Third, can such guarantees adapt to the frontier capacity of the feasible or realized profile family rather than its worst-case value over the entire cube, without knowing this complexity in advance?

\paragraph*{Acknowledgements.}
Thanks to Sanmi Koyejo for stimulating discussions.

\printbibliography

\clearpage

\appendix
\counterwithin{figure}{section}
\counterwithin{table}{section}

\phantomsection
\addcontentsline{toc}{section}{Appendices}
\etocsetlocaltop.toc{part}
\etocdepthtag.toc{appendixstart}

\begingroup

\etocsettagdepth{appendixstart}{subsection}
\etocsetnexttocdepth{subsection}
\YAAppendixContents

\endgroup

\section{Adaptive data analysis and differential privacy tools}
\label{app:infrastructure}

A query $q:\mathcal Z^n\to\R$ is $1/n$-sensitive if
$|q(x)-q(x')|\le1/n$ whenever $x,x'$ differ in one entry, and its population value is
$q(P)\mathrel{\coloneqq}\E_{X\sim P^n}[q(X)]$.
For background on differential privacy, we refer to~\cite{dwork2014foundations}.

We first need a classical result that transfers accuracy on the hidden sample to
accuracy for population values when the interaction is differentially private.

\begin{lemma}[Low-sensitivity transfer~{\cite[Theorem~3.3]{bassily2016stability}}]
\label{fact:transfer}
Let $0<a,b<1/10$.  Suppose a mechanism answering $m\ge1$ adaptively chosen $1/n$-sensitive queries $q_t$ with answers $y_t$ is
$(a/64,ab/32)$-differentially private and, for every fixed sample $z\in\mathcal Z^n$, satisfies
\begin{equation*}
    \Prb\!\left[\max_{t\le m}|y_t-q_t(z)|\le a/8\right]\ge1-ab/16.
\end{equation*}
Then, for every $P$ and $S\sim P^n$, we have
\begin{equation*}
    \Prb\!\left[\max_{t\le m}|y_t-q_t(P)|\le a\right]\ge1-b.
\end{equation*}
The first probability is over the interaction's randomness; the second also includes the sample.
\end{lemma}

For the dimension-free fallback, we use prior work's bounded noise technique to control the largest
error over all adaptive releases without an additional logarithmic factor in
the number of queries.

\begin{lemma}[Bounded-noise adaptive releases~{\cite[Theorem~1]{dagan2022bounded}}]
\label{fact:bounded-noise}
There is an absolute constant $C\ge1$ such that the following holds.
Let $k\ge3$, $n\ge1$ be integers, $\Delta>0$, $0<\eta\le1$, and
\begin{equation*}
    \delta_*(k)\coloneqq
    \exp\!\left(-\frac{k}{(\log k)^2(\log\log k)^4}\right)
    \le\delta\le\frac12.
\end{equation*}
There is an $(\eta,\delta)$-differentially private mechanism answering $k$
adaptively chosen queries $q_t:\mathcal Z^n\to\R$ satisfying
$|q_t(x)-q_t(x')|\le\Delta$ whenever $x,x'$ differ in one entry.
For every fixed sample $z\in\mathcal Z^n$ and every adaptive analyst,
its answers $y_t$ satisfy, with probability one,
\begin{equation*}
    \max_{t\le k}|y_t-q_t(z)|
    \le\frac{C\Delta\sqrt{k\log(1/\delta)}}{\eta}.
\end{equation*}
\end{lemma}

To release retained evaluation profiles, we use Gaussian noise.  The next
lemma gives standard controls of the privacy cost and the largest coordinate error over
all releases.

\begin{lemma}[Adaptive Gaussian releases]
\label{lem:gaussian}
There is an absolute constant $c_1\ge1$ such that the following holds.
Suppose an interactive mechanism releases at most $L$ adaptively chosen
vectors $q_t(S)\in\R^p$, each having Euclidean sensitivity at most
$\Delta_2$.  For $0<\eta,\delta,\zeta<1/2$, set
\begin{equation}
    \sigma
    \coloneqq
    \frac{
        c_1\Delta_2\sqrt{L\log(c_1/\delta)}
    }{\eta}.
    \label{eq:gaussian-scale}
\end{equation}
If each released vector is $\widetilde q_t\coloneqq q_t(S)+\xi_t$, where
$\xi_t\sim\mathcal N(0,\sigma^2 I_p)$ independently across releases, then the joint transcript is
$(\eta,\delta)$-differentially private.  Moreover, with probability at
least $1-\zeta$,
\begin{equation}
    \max_{t\le L}
    \|\widetilde q_t-q_t(S)\|_\infty
    \le
    \sigma\sqrt{2\log\!\left(\frac{2pL}{\zeta}\right)}.
    \label{eq:gaussian-max-noise}
\end{equation}
\end{lemma}
\begin{proof}
We use the zero-concentrated variant of differential privacy for convenience (zCDP)~\cite{bun2016concentrated}.
A release with covariance $\sigma^2I_p$ and Euclidean sensitivity
$\Delta_2$ is $\rho_0$-zCDP with
$\rho_0\mathrel{\coloneqq}\frac{\Delta_2^2}{2\sigma^2}$.
Adaptive composition of at most $L$ releases is therefore
$\rho$-zCDP with $\rho\mathrel{\coloneqq}L\rho_0$.  A $\rho$-zCDP mechanism is
$\bigl(\rho+2\sqrt{\rho\log(1/\delta)},\delta\bigr)$-differentially
private.  With $\sigma$ as in
\eqref{eq:gaussian-scale}, choosing the absolute constant $c_1$
sufficiently large makes this at most $(\eta,\delta)$.

There are at most $pL$ Gaussian coordinates in the transcript.  A
union bound and the Gaussian tail inequality give
\begin{equation*}
    \Prb\!\left[
        \max_{t\le L}\|\xi_t\|_\infty
        >
        \sigma\sqrt{2\log\!\left(\frac{2pL}{\zeta}\right)}
    \right]
    \le\zeta,
\end{equation*}
which proves \eqref{eq:gaussian-max-noise}.
\end{proof}

The update decisions require only threshold comparisons.  Sparse vector
is a classical technique to make these comparisons privately, with a cost controlled by the cap on
positive reports.

\begin{lemma}[Sparse vector~{\cite[Theorem~3.25 and proof of Theorem~3.24]{dwork2014foundations}}]
\label{fact:sv}
Consider at most $k$ adaptively chosen real-valued queries of sensitivity $1/n$, a fixed threshold $\tau$, and a cap $L\ge1$ on positive reports.  For every $0<\eta,\delta,\zeta<1/2$, there is an $(\eta,\delta)$-differentially private mechanism that halts only upon its $L$-th positive report and such that, with probability at least $1-\zeta$, every processed report satisfies
\begin{equation}
    \text{negative}\ \Longrightarrow\ q_t(S)\le\tau+u,
    \qquad
    \text{positive}\ \Longrightarrow\ q_t(S)\ge\tau-u,
    \label{eq:sv-reports}
\end{equation}
provided
\begin{equation}
    u\ge
    \frac{c_2\sqrt{L\log(c_2/\delta)}}{n\eta}
    \log\!\left(\frac{c_2kL}{\zeta}\right)
    \label{eq:sv-condition}
\end{equation}
for some absolute constant $c_2\ge1$.
\end{lemma}

Restarting AboveThreshold after each positive report and composing at most $L$ restarts gives Lemma~\ref{fact:sv}~\cite[Theorems~3.24--3.26]{dwork2014foundations}.  The proof of Theorem~3.24 in~\cite{dwork2014foundations} gives accuracy of processed reports without the sparsity assumption used to prevent early halting.  Composition contributes $\sqrt{L\log(1/\delta)}$; the Laplace tail bound and a union bound over threshold tests contribute the final logarithm in \eqref{eq:sv-condition}, outside the square root.

For our lower bounds, we need the converse statistical limitation:
estimating $m$ adaptive query means at sufficiently small constant error and
failure probability requires at least order $\sqrt m$ samples.

\begin{lemma}[Adaptive-query lower bound~{\cite[Theorem~35]{steinke2015interactive}}]
\label{fact:aq-lower}
Recall that $\Naq(m,\alpha,\beta)$ is the minimax sample size for estimating
the population means of $m$ adaptively chosen $[-1,1]$-valued queries from one
i.i.d.\ sample, with simultaneous error at most $\alpha$ and failure probability
at most $\beta$.

There are absolute constants $c_3,c_4,c_5\in(0,1)$ such that, for every $m\ge1$,
\begin{equation}
    \Naq(m,c_3,c_4)\ge c_5\sqrt m.
    \label{eq:aq-lower}
\end{equation}
\end{lemma}

Finally, the population value of an empirical-minimum query is the
expectation of a minimum, whereas leaderboard validity concerns the minimum
of population means.  The following lemma bounds this difference and records
the query's sensitivity.
For fixed functions $h_1,\ldots,h_t:\mathcal Z\to[0,1]$, define
$
    Q(S)\mathrel{\coloneqq}\min_{s\le t}\frac1n\sum_{i=1}^nh_s(Z_i),
$
and
$
    v(P)\mathrel{\coloneqq}\min_{s\le t}\E_Ph_s(Z).
$
\begin{lemma}[Empirical-minimum bias]
\label{lem:bias}
For every $t\ge1$, we have
\begin{equation}
    0\le v(P)-\E_{S\sim P^n}Q(S)
    \le\sqrt{\frac{\log(2t)}{2n}}.
    \label{eq:bias}
\end{equation}
Moreover, $Q$ is $1/n$-sensitive.
\end{lemma}

\begin{proof}
Let $\mu_s\mathrel{\coloneqq}\E_Ph_s(Z)$ and
$\overline h_s\mathrel{\coloneqq}\frac{1}{n}\sum_i h_s(Z_i)$.  Concavity of the minimum gives
\begin{equation*}
    \E_{S\sim P^n}Q(S)=
    \E\min_s\overline h_s
    \le\min_s\E\overline h_s
    =\min_s\mu_s
    =v(P).
\end{equation*}
Also,
$
    \min_s\mu_s-\min_s\overline h_s
    \le\max_{s\le t}|\overline h_s-\mu_s|.
$
Hoeffding's inequality gives
$\E\exp(\lambda(\overline h_s-\mu_s))\le
\exp(\lambda^2/(8n))$.  Hence, for every $\lambda>0$,
\begin{align*}
 \E\max_{s\le t}|\overline h_s-\mu_s|
 \le\frac1\lambda\log\!\left(
   \sum_{s\le t}\sum_{\sigma\in\{-1,1\}}
   \E e^{\lambda\sigma(\overline h_s-\mu_s)}
 \right)
 \le\frac{\log(2t)}\lambda+\frac{\lambda}{8n}.
\end{align*}
Optimizing at $\lambda=\sqrt{8n\log(2t)}$ proves \eqref{eq:bias}.  Replacing one sample changes every empirical mean by at most $1/n$, and
$|\min_s a_s-\min_s b_s|\le\max_s|a_s-b_s|$, so $Q$ is $1/n$-sensitive.
\end{proof}

\section{Additional geometry of frontier capacity}
\label{app:geometry}

\subsection{Properties of frontier capacity}
\label{app:frontier-properties}

Recall that $\mathsf P_\infty(K,\gamma)$ is the largest cardinality of a subset
of $K$ with pairwise $\ell_\infty$ distances at least $\gamma$.
The $\ell_1$ covering number of $\mathcal W$ at radius $r>0$ is
\begin{equation*}
    \mathcal N_1(\mathcal W,r)
    \coloneqq\min\left\{|V|:V\subseteq\mathcal W,\,
    \sup_{w\in\mathcal W}\inf_{v\in V}\|w-v\|_1\le r\right\}.
\end{equation*}

\begin{proposition}[Geometric upper bounds]
\label{prop:frontier-properties}
Let $K\subseteq[0,1]^d$ be nonempty, let $\mathcal W\subseteq\Delta_d$ be
nonempty and compact, and let $\gamma>0$.
We have
\begin{equation*}
    \Mcap_{\mathcal W}(K,\gamma)
    \le
    \min\left\{
        \mathsf P_\infty(K,\gamma),\;
        1+\left(1+\frac{2}{\gamma}\right)\mathcal N_1(\mathcal W,\gamma/2)
    \right\}.
\end{equation*}
Moreover,
$\mathsf P_\infty(K,\gamma)\le(1+\lceil1/\gamma\rceil)^d$.
\end{proposition}

\begin{proof}
Let $b_1,\ldots,b_m\in K$ be
$\gamma$-convexly separated over $\mathcal W$.
For the packing bound, \eqref{eq:dual-certificate} gives
$\gamma\le\langle w_t,b_s-b_t\rangle\le\|b_s-b_t\|_\infty$ for $s<t$,
so $m\le\mathsf P_\infty(K,\gamma)$.
Partitioning each coordinate interval $[0,1]$ into $1+\lceil1/\gamma\rceil$
intervals of length strictly smaller than $\gamma$ gives the stated cube bound.

For the covering bound, assign each certificate $w_t$, $t\ge2$, to a point $v$
of a smallest $\gamma/2$-net of $\mathcal W$ in $\ell_1$.
If $s<t$ are assigned to the same $v$, then
$\langle v,b_s-b_t\rangle\ge\gamma-\|w_t-v\|_1\ge\gamma/2$.
Since the scalar scores lie in $[0,1]$, each net point accounts for at most
$1+2/\gamma$ profiles; the extra $1$ accounts for $b_1$.
Taking the supremum over such sequences proves the capacity bound.
\end{proof}

\begin{remark}[Exponential looseness in~\eqref{eq:sep-packing-upper}]
\label{rem:simultaneous-looseness}
Fix $0<\gamma<1$ and $m\ge2$. Let $d=2m$ and set
\begin{equation*}
    \mathcal W=\Delta_m\times\{0_m\},
    \qquad
    K=\{(a\mathbf 1_m,z):a\in[0,1],\ z\in\{0,1\}^m\}.
\end{equation*}
Every $w\in\mathcal W$ assigns score $a$ to $(a\mathbf 1_m,z)$, so each
successive improvement lowers $a$ by at least $\gamma$ and
$\Mcap_{\mathcal W}(K,\gamma)\le1+\lfloor1/\gamma\rfloor$.
Fixing $a=0$ gives $\mathsf P_\infty(K,\gamma)\ge2^m$.
For the covering number, project $\Delta_m$ onto its first $m-1$ coordinates:
the resulting simplex has volume $1/(m-1)!$, while each projected
$\ell_1$ ball of radius $\gamma/2$ has volume at most
$\gamma^{m-1}/(m-1)!$.
Thus $\mathcal N_1(\mathcal W,\gamma/2)\ge\gamma^{-(m-1)}$.
At fixed $\gamma$, both terms in \eqref{eq:sep-packing-upper} therefore grow
exponentially in $m=d/2$, while frontier capacity stays constant in $d$.
\end{remark}

\begin{remark}[Relation to convexified packing]
\label{rem:convexified-packing}
Frontier capacity is closely related to the notion of convexified
packing~\cite{artstein2004convexified}.  Let $B\subseteq\R^d$ be an
origin-symmetric convex body.  A sequence $b_1,\ldots,b_m\in K$ is a
convexified $\gamma B$-packing if, for every $t\ge2$,
\begin{equation}
    (b_t+\gamma\,\operatorname{int}B)
    \cap \operatorname{conv}\{b_s:s<t\}
    =\varnothing.
\end{equation}
By the separating hyperplane theorem, this is equivalent to requiring, for
each $t$, a vector $w_t$ such that
$
    \langle w_t,b_s-b_t\rangle\ge\gamma, s<t,
$
after rescaling $w_t$ so that
$\sup_{u\in B}\langle w_t,u\rangle\le1$.
The set of all such normalized vectors is the polar body
$
    B^\circ
    :=\{w\in\R^d:\sup_{u\in B}\langle w,u\rangle\le1\}.
$
Thus convexified packing allows any certificate $w_t\in B^\circ$, whereas
frontier capacity uses the same sequential certificate condition but requires
$w_t\in\mathcal W\subseteq\Delta_d$.  In our setting, every new frontier point
must therefore be exposed by an allowed benchmark comparison.
\end{remark}

\subsection{The simplex construction}
\label{app:simplex-construction}
To prove the exponential lower bound on frontier capacity over the full simplex below, the challenge is to construct many profiles such that each can be exposed by a weight in $\Delta_d$ with a dimension-independent margin. An affine encoding does not suffice uniformly over $q$: already for $d=2$, its $q$ images lie on a line, so a linear score can strictly expose only the two endpoints. For $q=2$ though, a constant-weight binary code gives a simpler affine construction. 

We therefore map well-separated $q$-ary codewords to a convex quadratic graph. The tangent plane at $x$ supports the graph from below with a gap proportional to $\|x-y\|_2^2$; choosing the graph with nonpositive slopes makes its upward normal nonnegative and hence normalizable into $\Delta_d$. A standard $q$-ary Gilbert packing then supplies exponentially many codewords whose Hamming separation yields the required exposure margin.

\qarycode*

\begin{proof}
We first construct a large $q$-ary code by the standard Gilbert greedy-packing argument~\cite{gilbert1952comparison}. Put $p\coloneqq d-1$.  Greedily select a vector from
$\{0,\ldots,q-1\}^p$ and remove all remaining vectors within Hamming distance
$\lfloor p/4\rfloor$.  Each step removes at most
\begin{equation*}
    B_{p,q}\coloneqq\sum_{j\le p/4}\binom pj(q-1)^j
    \le\exp\!\left(pH(1/4)+\frac p4\log(q-1)\right),
\end{equation*}
where $H(u)\coloneqq-u\log u-(1-u)\log(1-u)$.  The inequality uses
$(q-1)^j\le(q-1)^{p/4}$ and the entropy bound
$\sum_{j\le p/4}\binom pj\le e^{pH(1/4)}$, obtained by expanding
$(1/4+3/4)^p$ and bounding each factor
$(1/4)^j(3/4)^{p-j}$ below by $e^{-pH(1/4)}$ for $j\le p/4$.
The removed balls cover all $q^p$ vectors, so the selected set $\mathcal C$ has
at least $q^p/B_{p,q}\ge q^{a(q)p}$ elements, where
\begin{equation*}
    a(q)\coloneqq
    \frac{\log q-H(1/4)-\frac14\log(q-1)}{\log q}.
\end{equation*}
$a(2)>0$, and for $q\ge3$,
$a(q)\ge3/4-H(1/4)/\log3>0$.  Thus one absolute constant $c_0>0$
gives $|\mathcal C|\ge\lfloor q^{c_0p}\rfloor$ for every $p\ge1$ and $q\ge2$.
Distinct selected vectors have Hamming distance greater than $\lfloor p/4\rfloor$, hence at least $p/4$.

Map $a\in\mathcal C$ to $x(a)\mathrel{\coloneqq}a/(q-1)\in[0,1]^p$.  For $x\in x(\mathcal C)$, define
\begin{equation*}
    b_i(x)\mathrel{\coloneqq}\frac14+\frac{x_i}{2}\quad(i\le p),
    \qquad
    b_d(x)\mathrel{\coloneqq}\frac12+\frac1{8p}\sum_{i=1}^p(1-x_i)^2,
\end{equation*}
and
\begin{equation*}
    D(x)\mathrel{\coloneqq}1+\frac1{2p}\sum_{i=1}^p(1-x_i),
    \qquad
    w_d(x)\mathrel{\coloneqq}D(x)^{-1},
    \qquad
    w_i(x)\mathrel{\coloneqq}\frac{1-x_i}{2pD(x)}.
\end{equation*}
Then $b_i(x)\in[1/4,3/4]$ for $i\le p$,
$b_d(x)\in[1/2,5/8]$, all coordinates of $w(x)$ are nonnegative, and
\begin{equation*}
    \sum_{i=1}^pw_i(x)+w_d(x)
    =\frac{\frac1{2p}\sum_i(1-x_i)+1}{D(x)}=1.
\end{equation*}
Thus $b(x)\in[1/4,3/4]^d\subset[0,1]^d$ and $w(x)\in\Delt_d$.  Also $D(x)\le3/2$, so $w_d(x)\ge2/3$.  Direct expansion gives, for all $x,y\in[0,1]^p$,
\begin{align*}
 \langle w(x),b(y)-b(x)\rangle
 &=\frac1{D(x)}\left[
   \frac1{4p}\sum_{i=1}^p(1-x_i)(y_i-x_i)
   +\frac1{8p}\sum_{i=1}^p
      \bigl((1-y_i)^2-(1-x_i)^2\bigr)
 \right]\\
 &=\frac{w_d(x)}{8p}\|y-x\|_2^2.
\end{align*}
Distinct vectors in $\mathcal C$ differ in at least $p/4$ coordinates, and every differing coordinate changes by at least $1/(q-1)$.  Hence, for $y\neq x$, we have
\begin{equation*}
    \langle w(x),b(y)-b(x)\rangle
    \ge\frac{2/3}{8p}\cdot\frac{p}{4(q-1)^2}
    =\frac1{48(q-1)^2}.
\end{equation*}
Thus, in any ordering, each new $b(x)$ satisfies Eq.~\eqref{eq:dual-certificate} against every predecessor, proving the lemma.
\end{proof}

\section{Lower bound constructions and proofs}
\label{app:lower-proofs}

\subsection{Nonadaptive selection}

For completeness, we include the nonadaptive sample complexity lower bound, adapting Blum--Hardt~\cite{blum2015ladder} to our multi-criteria setting.

\begin{restatable}[Nonadaptive selection lower bound]{proposition}{loglower}
\label{thm:log-lower}
Let $k\ge4$, let $\mathcal W\subseteq\Delt_d$ be nonempty, let
$0<\varepsilon\le1/16$, and let $0<\beta\le1/4$.  Then
\begin{equation}
    \Npw(k,\mathcal W,\varepsilon,\beta)
    \ge\frac{\log k}{256\varepsilon^2}.
    \label{eq:log-lower}
\end{equation}
\end{restatable}
\begin{proof}
The argument follows Blum--Hardt's mean estimation reduction~\cite{blum2015ladder}: we hide one
slightly lower-mean coordinate among $k$ otherwise fair coordinates and
submit the coordinate projections in order.  An accurate weighted leaderboard reveals the hidden coordinate from the round at which its
value drops, then Fano's inequality limits how reliably $n$ samples
can identify that coordinate.

Let $\mathcal Z\coloneqq\{0,1\}^k$, and write
$Z=(Z^{(1)},\ldots,Z^{(k)})\in\mathcal Z$.  For $i\in[k]$, let $P_i$
be the product distribution under which
\begin{equation*}
    \E_{P_i} Z^{(i)}=\frac12-4\varepsilon,
    \qquad
    \E_{P_i} Z^{(t)}=\frac12
    \quad (t\ne i).
\end{equation*}
Fix any $\bar w\in\mathcal W$.  Use the fixed model class of maps
$f:\mathcal Z\to[0,1]^d$ and evaluation map
$\phi(f,z)\coloneqq f(z)$.  Submit the nonadaptive models
$
    f_t(z)\coloneqq z^{(t)}\one,
 t\in[k],
$
and use $w_t=\bar w$ throughout.  Since
$\langle \bar w,\one\rangle=1$, under $P_i$ the best-so-far value is
\begin{equation*}
    V_t^{P_i}(\bar w)=
    \begin{cases}
      1/2, & t<i,\\
      1/2-4\varepsilon, & t\ge i.
    \end{cases}
\end{equation*}
Thus the location of the downward jump identifies $i$.  Given the
transcript, define
\begin{equation*}
    \widehat i\coloneqq
    \begin{cases}
      \min\{t\in[k]:\widehat v_t<1/2-2\varepsilon\},
      &\text{if this set is nonempty},\\
      1, &\text{otherwise}.
    \end{cases}
\end{equation*}
On the event
$
    \max_{t\le k}
    \bigl|\widehat v_t-V_t^{P_i}(\bar w)\bigr|
    \le\varepsilon,
$
we have $\widehat i=i$.

Now let $J$ be uniform on $[k]$ and draw $S\sim P_J^n$.  By validity of the weighted leaderboard mechanism,
$\widehat i$ identifies the location $J$ of the biased coordinate with
error probability at most $\beta$.

Let $\Pi$ denote the transcript, and let $P_0$ be the product
distribution on $\mathcal Z$ under which every coordinate is
$\Ber(1/2)$.  Since
$
    J\rightarrow S\rightarrow\Pi\rightarrow\widehat i
$
is a Markov chain, data processing and comparison with $P_0^n$ give
\begin{equation*}
    I(J;\widehat i)
    \le I(J;\Pi)
    \le I(J;S)
    \le \frac1k\sum_{i=1}^k
        \KL(P_i^n\|P_0^n)
    =
    n\,\KL\!\left(
        \Ber(1/2-4\varepsilon)
        \middle\|
        \Ber(1/2)
    \right).
\end{equation*}
Here $I(\cdot;\cdot)$ denotes mutual information and $\KL$ denotes
Kullback--Leibler divergence.  Using
$\KL(Q\|R)\le\chi^2(Q\|R)$, we have
$
    \KL\!\left(
        \Ber(1/2-4\varepsilon)
        \middle\|
        \Ber(1/2)
    \right)
    \le 64\varepsilon^2.
$
Fano's inequality thus gives
\begin{equation*}
    \beta
    \ge
    1-\frac{I(J;\widehat i)+\log 2}{\log k}
    \ge
    1-\frac{64n\varepsilon^2+\log 2}{\log k}.
\end{equation*}
Since $k\ge4$ and $\beta\le1/4$,
$
    (1-\beta)\log k-\log 2
    \ge \frac14\log k,
$
and hence
$
    n\ge\frac{\log k}{256\varepsilon^2}.
$
\end{proof}

\subsection{Realization by binary classification}

The population profiles used in both reductions can be realized by  $d$-output binary classifiers with coordinatewise $0$--$1$ loss.  Below, $P_0$ is the population distribution in the adaptive-query problem; augmenting each observation with independent randomness converts it into a classification distribution.

\begin{proposition}[Fixed-classification realization]
\label{prop:classification-realization}
Let $P_0$ be any probability distribution on a space
$\mathcal Z_0$, and let $d\ge1$.  There exists a distribution
$P=P(P_0,d)$ on
$
    \mathcal X\times \{0,1\}^d,
    \mathcal X
    \mathrel{\coloneqq}
    \mathcal Z_0\times\{0,1\}^d\times[0,1]^d,
$
such that:
\begin{enumerate}\setlength{\itemsep}{0pt}
    \item For every $0<\rho\le1/2$, every $b\in[0,1]^d$, and every
    function $g:\mathcal Z_0\to[-1,1]$, there is a classifier
    $f_{\rho,b,g}:\mathcal X\to\{0,1\}^d$ satisfying
    \begin{equation*}
        R_P(f_{\rho,b,g})
        =
        (1-2\rho)b
        +
        \rho\bigl(1+\E_{P_0}[g(Z_0)]\bigr)\one .
    \end{equation*}
    \item The distribution $P$ and coordinatewise $0$--$1$ loss are
    independent of $\rho,b,g$.  Moreover, an $n$-sample from $P$ can be
    generated from an $n$-sample from $P_0$ by augmenting each observation
    with independent randomness.
\end{enumerate}
\end{proposition}
\begin{proof}
Draw $Z_0\sim P_0$ and, independently, vectors $R,U$ with mutually independent coordinates
$R_j\sim\operatorname{Ber}(1/2)$ and $U_j\sim\operatorname{Unif}[0,1]$, $j\in[d]$.
Set $X\coloneqq(Z_0,R,U)$ and $Y\coloneqq R$; their joint law is $P\coloneqq P(P_0,d)$.

Fix $0<\rho\le1/2$, $b\in[0,1]^d$, and a function
$g:\mathcal Z_0\to[-1,1]$, and define
\begin{equation*}
    \theta_j(z)
    \mathrel{\coloneqq}
    (1-2\rho)b_j+\rho(1+g(z)).
\end{equation*}
Since $b_j\in[0,1]$ and $g(z)\in[-1,1]$,
$0\le\theta_j(z)\le1$.
Define the model's $j$-th prediction coordinate
\begin{equation*}
    f_{\rho,b,g,j}(z,r,u)
    \coloneqq
    r_j\oplus\mathbf 1\{u_j\le\theta_j(z)\}.
\end{equation*}
Then, conditional on $Z_0=z$, we can write the conditional risk as
\begin{align*}
    \Prb\!\left(f_{\rho,b,g,j}(X)\ne Y_j\mid Z_0=z\right)
    =
    \Prb\!\left(U_j\le\theta_j(z)\right)
    =
    \theta_j(z)
    =
    (1-2\rho)b_j+\rho(1+g(z)).
\end{align*}
Taking expectation over $Z_0$ gives
$
    R_P(f_{\rho,b,g})
    =
    (1-2\rho)b
    +
    \rho\bigl(1+\E_{P_0}[g(Z_0)]\bigr)\one.
$
Finally, $R$ and $U$ are generated independently conditional on $Z_0$,
so an $n$-sample from $P$ is obtained by augmenting each of $n$ samples
from $P_0$ with fresh independent draws of $R$ and $U$.
\end{proof}

\subsection{Winner-only feedback}
\label{app:winner-only}

At round $t$, a winner-only mechanism returns an index
$\widehat s_t\in[t]$.  It is $(\varepsilon,\beta)$-valid for selection if
\begin{equation*}
    \Prb\!\left[
      \langle w_t,R_P(f_{\widehat s_t})\rangle
      \le V_t^P(w_t)+\varepsilon
      \quad\text{for every }t\le k
    \right]\ge1-\beta.
\end{equation*}
Let $\Nsel(k,\mathcal W,\varepsilon,\beta)$ be the least sample size for such a mechanism, and write
$\Nsel(k,d,\varepsilon,\beta)\mathrel{\coloneqq}
\Nsel(k,\Delt_d,\varepsilon,\beta)$.

\selectionlower*
\begin{proof}
    We start with a binary-search-based reduction from adaptive-query  estimation, mirroring the proof of Theorem~\ref{thm:capacity-lower}.

\emph{Binary-search reduction.}
Let $0<\gamma\le1$,
$0<\varepsilon\le\gamma/16$, and $0<\beta<1$, and put
\begin{equation*}
    h\coloneqq\left\lceil\log_2\frac{\gamma}{8\varepsilon}\right\rceil,
    \qquad
    m\coloneqq\left\lfloor
        \frac{\min\{k/2,\Mcap_{\mathcal W}(d,\gamma)\}}{h}
    \right\rfloor.
\end{equation*}
If $m\ge1$, we claim that
\begin{equation}
    \Nsel(k,\mathcal W,\varepsilon,\beta)
    \ge\Naq\!\left(m,\frac{16\varepsilon}{\gamma},\beta\right).
    \label{eq:selection-reduction}
\end{equation}
We estimate each adaptive query mean by binary search.  Each comparison
submits the query and a constant threshold at a new separated evaluation
profile, so earlier submissions cannot win the comparison.

Suppose an $(\varepsilon,\beta)$-valid winner-only mechanism uses $n$ samples.
By construction, $h,m\ge1$, $2mh\le k$, and
$mh\le\Mcap_{\mathcal W}(d,\gamma)$.
Choose a sequence $b_1,\ldots,b_{mh}$ that is $\gamma$-convexly separated over $\mathcal W$, with certificate weights $w_2,\ldots,w_{mh}$ and any $w_1\in\mathcal W$.
These profiles and weights are fixed before the interaction.
Set $\rho\coloneqq\gamma/8$ and $\delta\coloneqq\varepsilon/\rho=8\varepsilon/\gamma$.
The assumption $\varepsilon\le\gamma/16$ ensures that the separation between comparisons exceeds the selection error.

At adaptive-query round $t$, receive $g_t:\mathcal Z\to[-1,1]$, write
$\mu_t\coloneqq\E_P[g_t(Z)]$, and initialize
$I_{t,0}\coloneqq[L_{t,0},U_{t,0}]\coloneqq[-1,1]$.
For each of the $h$ comparison steps $q=1,\ldots,h$, set
$j\coloneqq(t-1)h+q$ and
$\tau_{t,q}\coloneqq(L_{t,q-1}+U_{t,q-1})/2$.
In two consecutive leaderboard rounds, both with weighting $w_j$, submit the query model and threshold model with evaluation functions
\begin{equation*}
    z\longmapsto(1-2\rho)b_j+\rho(1+g_t(z))\one,
    \qquad
    z\longmapsto(1-2\rho)b_j+\rho(1+\tau_{t,q})\one,
\end{equation*}
respectively, and retain only the second returned identity.

On the event that every returned model is $\varepsilon$-optimal, this identity belongs to the current pair.
Indeed, an earlier model has population profile
$(1-2\rho)b_i+\rho(1+\nu)\one$ for some $i<j$ and $\nu\in[-1,1]$.
Writing $a\coloneqq\min\{\mu_t,\tau_{t,q}\}$, its excess loss over the better current model is at least
\begin{align*}
    (1-2\rho)\langle w_j,b_i-b_j\rangle+\rho(\nu-a)
    \ge(1-2\rho)\gamma-2\rho
    =\frac{\gamma(3-\gamma)}4
    \ge\frac\gamma2>\varepsilon.
\end{align*}
The current pair has losses $c_j+\rho\mu_t$ and $c_j+\rho\tau_{t,q}$, where
$c_j\coloneqq(1-2\rho)\langle w_j,b_j\rangle+\rho$.
If the query model, i.e., the first model in the pair, is returned,
$\varepsilon$-optimality gives $\mu_t\le\tau_{t,q}+\delta$; if the threshold model is returned, it gives $\mu_t\ge\tau_{t,q}-\delta$.
Therefore update
\begin{equation*}
    I_{t,q}=[L_{t,q},U_{t,q}]\coloneqq
    \begin{cases}
        I_{t,q-1}\cap[-1,\tau_{t,q}+\delta],&\text{if the query model is returned},\\
        I_{t,q-1}\cap[\tau_{t,q}-\delta,1],&\text{if the threshold model is returned}.
    \end{cases}
\end{equation*}
If an earlier model is returned outside the validity event, leave the interval unchanged.
On the validity event, every interval contains $\mu_t$ and
$|I_{t,q}|\le|I_{t,q-1}|/2+\delta$.
Since $|I_{t,0}|=2$, the midpoint $y_t$ of $I_{t,h}$ satisfies
\begin{equation*}
    |y_t-\mu_t|\le2^{-h}+\delta\le\frac{16\varepsilon}{\gamma}.
\end{equation*}
The adaptive-query analyst can use these answers to choose later queries because every $y_t$ is computed from the winner-only transcript.
Thus the same $n$ samples answer $m$ adaptive queries at the stated accuracy and confidence.
Finally, Proposition~\ref{prop:classification-realization} realizes both models in each pair by taking $g=g_t$ and $g\equiv\tau_{t,q}$, respectively, under one fixed binary classification distribution.

\emph{Final step.}
Choose $C=16/c_3$, $\varepsilon_0\le\min\{1/C,1/16\}$, and
$\beta_0\le\min\{c_4,1/4\}$, using the constants from
Lemma~\ref{fact:aq-lower}.  With $\gamma=C\varepsilon$, the reduction has
accuracy $16/C=c_3$ and a constant overhead
$h=\lceil\log_2(C/8)\rceil\ge1$.
Write $r\coloneqq\min\{k,\Mcap_{\mathcal W}(d,C\varepsilon)\}$.
If $r\ge4h$, then
\begin{equation*}
    m\ge\left\lfloor\frac{r}{2h}\right\rfloor\ge\frac{r}{4h},
\end{equation*}
so \eqref{eq:selection-reduction} and Lemma~\ref{fact:aq-lower} give
$\Nsel(k,\mathcal W,\varepsilon,\beta)\ge c_5\sqrt{r}/(2\sqrt h)$.
For $r<4h$, the same conclusion holds after decreasing the  constant, since sample sizes are positive integers. This proves the theorem.
\end{proof}

\section{Frontier Ladder: proofs and stronger release guarantees}
\label{app:upper-proofs}

\subsection{Dimension-free fallback}

At round $t$, after receiving $(f_t,w_t)$, define the empirical lower-envelope query
\begin{equation}
    Q_t(X)\mathrel{\coloneqq}\min_{s\le t}\frac1n\sum_{i=1}^n
      \langle w_t,\phi(f_s,X_i)\rangle.
    \label{eq:Qt}
\end{equation}

\genericupper*
\begin{proof}
For sufficiently large $k$, release $Q_t(S)$ using the mechanism in Lemma~\ref{fact:bounded-noise}.
Apply Lemma~\ref{fact:transfer} with $a=\varepsilon_0/2$, $b=\beta_0/2$.
It is enough that the transcript be
$(\eta,\delta)=\left(\frac{\varepsilon_0}{128},\frac{\varepsilon_0\beta_0}{128}\right)$
differentially private and that, for every sample, all answers differ from
$Q_t(S)$ by at most $\varepsilon_0/16$, except with probability
$\varepsilon_0\beta_0/64$.

Conditioned on a fixed preceding transcript, all submitted models and the current
weight vector in \eqref{eq:Qt} are fixed.  Lemma~\ref{lem:bias} therefore shows
that the current query is $1/n$-sensitive.  Since $0<\eta\le1$,
$0<\delta\le1/2$, and $\delta_*(k)\to0$, Lemma~\ref{fact:bounded-noise}
applies for every $k\ge k_0$, where $k_0\ge3$ depends only on
$\varepsilon_0,\beta_0$.
With sensitivity $\Delta=1/n$, it gives the required
privacy and sample accuracy, the latter with probability one, whenever
\begin{equation*}
    n\ge C\sqrt{k}
\end{equation*}
for sufficiently large $C$.
Lemma~\ref{fact:transfer} then gives, with failure probability at most $\beta_0/2$,
\begin{equation*}
    \max_{t\le k}|\widehat v_t-Q_t(P)|\le\varepsilon_0/2,
    \qquad
    Q_t(P)\mathrel{\coloneqq}\E_{X\sim P^n}Q_t(X).
\end{equation*}
Conditioned on the transcript before round $t$, the functions
$z\mapsto\langle w_t,\phi(f_s,z)\rangle$, $s\le t$, are fixed and their values lie in $[0,1]$.
Lemma~\ref{lem:bias} gives
\begin{equation*}
    |Q_t(P)-V_t^P(w_t)|\le\sqrt{\frac{\log(2k)}{2n}}.
\end{equation*}
Increasing $C$ makes this at most $\varepsilon_0/2$, since $\log(2k)=\mathcal{O}(\sqrt{k})$.
The triangle inequality proves validity for $k\ge k_0$.
For $2\le k<k_0$, the Gaussian bound \eqref{eq:generic-upper-gaussian}
below is at most a constant depending only on $\varepsilon_0,\beta_0$ times
$\sqrt{k}$, since $k$ ranges over a finite set.  Increasing $C$ in
\eqref{eq:generic-upper} absorbs these cases and integer rounding for every $k\ge2$.
\end{proof}

For general $0<\varepsilon,\beta\le1/10$, the Gaussian mechanism gives,
with a universal constant $C\ge1$, for every nonempty compact
$\mathcal W\subseteq\Delta_d$ and every $k\ge2$,
\begin{equation}
    \Npw(k,\mathcal W,\varepsilon,\beta)
    \le\frac{C}{\varepsilon^2}
    \left[
      \sqrt{k\log\bigl(C/(\varepsilon\beta)\bigr)
        \log\bigl(Ck/(\varepsilon\beta)\bigr)}
      +\log(2k)
    \right].
    \label{eq:generic-upper-gaussian}
\end{equation}
Indeed, Lemma~\ref{lem:gaussian}, with $p=1$, $L=k$, $\Delta_2=1/n$,
$\eta=\varepsilon/128$, $\delta=\varepsilon\beta/128$, and
$\zeta=\varepsilon\beta/64$, gives the required privacy and sample accuracy
from the first term.  Lemma~\ref{fact:transfer} and Lemma~\ref{lem:bias}
then give validity by the same argument, using the second term to bound the bias.

\subsection{Frontier Ladder validity}

The following two lemmas focus on the geometric utility argument for Frontier Ladder.

\begin{lemma}[Accepted rounds are convexly separated]
\label{lem:packing}
Assume every sparse-vector report satisfies \eqref{eq:sv-reports} with tolerance $u$, and every accepted profile estimate $\widetilde R_t$ satisfies
$\|\widetilde R_t-\widehat R_t(S)\|_\infty\le u$.  Then the empirical profiles at positive rounds form a
$2u$-convexly separated sequence over $\mathcal W$ in $[0,1]^d$.  In particular, Frontier Ladder makes at most
$\mathsf{L}_{\mathcal W}(k,\varepsilon)-1$ positive reports, so its cap
$\mathsf{L}_{\mathcal W}(k,\varepsilon)$ is not reached.
\end{lemma}

\begin{proof}
Suppose round $t$ is positive.  Then
$\Gamma_t(S)\ge \tau-u=3u$.  The objective in \eqref{eq:frontier-novelty} is continuous and $\mathcal W$ is compact, so there exists $w\in\mathcal W$ such that, for every previously accepted round $s$,
\begin{equation*}
    \langle w,\widetilde R_s-\widehat R_t(S)\rangle\ge3u.
\end{equation*}
Since $\|\widetilde R_s-\widehat R_s(S)\|_\infty\le u$ and $w\in\Delt_d$,
\begin{equation*}
    \langle w,\widehat R_s(S)-\widehat R_t(S)\rangle\ge2u.
\end{equation*}
Equation~\eqref{eq:dual-certificate} proves the separation assertion.  The positive-round empirical profiles lie in $[0,1]^d$; since
$2u=\varepsilon/40$, their number is at most
\begin{equation*}
\min{\left\{k, \Mcap_{\mathcal W}\!\left(d,\frac{\varepsilon}{40}\right)\right\}}=\mathsf{L}_{\mathcal W}(k,\varepsilon)-1.
\end{equation*}
\end{proof}

\begin{lemma}[Sample accuracy]
\label{lem:sample-accuracy}
Recall that $F_A(w)\coloneqq\min_{a\in A}\langle w,a\rangle$ is the
lower envelope of the retained profiles.
Under the event in Lemma~\ref{lem:packing}, for every $t\le k$ and every $w\in\mathcal W$,
\begin{equation}
    -u\le F_{A_t}(w)-
      \min_{s\le t}\langle w,\widehat R_s(S)\rangle
    \le5u.
    \label{eq:sample-envelope}
\end{equation}
Consequently, we have
\begin{equation*}
    \max_{t\le k}|\widehat v_t-Q_t(S)|
    \le5u=\varepsilon/16.
\end{equation*}
\end{lemma}

\begin{proof}
Set
$H_t(w)\mathrel{\coloneqq}\min_{s\le t}\langle w,\widehat R_s(S)\rangle$, $H_0(w)\mathrel{\coloneqq}1$.
Every element of $A_t$ other than the initial profile $\one$ is within
$u$ in $\ell_\infty$ of an empirical profile from an accepted round, while
$\langle w,\one\rangle=1\ge H_t(w)$.  Hence every element of $A_t$ has score at least $H_t(w)-u$, so
$F_{A_t}(w)\ge H_t(w)-u$.

For the upper bound, use induction on $t$.  At $t=0$, the initialization
$A_0=\{\one\}$ gives $F_{A_0}(w)=1=H_0(w)$ for every $w\in\mathcal W$,
establishing the base case.  Suppose the bound holds at round $t-1$.
If round $t$ is positive, then
\begin{align*}
 F_{A_t}(w)
 =\min\{F_{A_{t-1}}(w),\langle w,\widetilde R_t\rangle\}
 \le\min\{H_{t-1}(w)+5u,
             \langle w,\widehat R_t(S)\rangle+u\}
 \le H_t(w)+5u.
\end{align*}
If round $t$ is negative, Lemma~\ref{fact:sv} gives
$\Gamma_t(S)\le \tau+u=5u$.  Thus, for every $w\in\mathcal W$,
\begin{equation*}
    F_{A_{t-1}}(w)
    \le\langle w,\widehat R_t(S)\rangle+5u.
\end{equation*}
Combining this with the induction hypothesis gives
\begin{align*}
 F_{A_t}(w)=F_{A_{t-1}}(w)
 \le\min\{H_{t-1}(w)+5u,
             \langle w,\widehat R_t(S)\rangle+5u\}
 =H_t(w)+5u.
\end{align*}
This proves \eqref{eq:sample-envelope}; evaluating at $w_t$ proves the final assertion.
\end{proof}

\geometricupper*
\begin{proof}
    Algorithm~\ref{alg:frontier-ladder} uses
$\eta_0=\frac{\varepsilon}{256}$, $\delta_0=\frac{\varepsilon\beta}{256}$,
$\zeta_0=\frac{\varepsilon\beta}{128}$.
Sparse vector is thus $(\eta_0,\delta_0)$-differentially private.
By Lemma~\ref{fact:sv}, its reports have tolerance
$u=\varepsilon/80$ except with probability $\zeta_0$, if for large enough $C$
\begin{equation*}
    n \ge C
    \frac{
      \sqrt{L\log(C/(\varepsilon\beta))}
      \log(CkL/(\varepsilon\beta))
    }{\varepsilon^2}.
\end{equation*}

There are at most $L$ accepted profile estimates before the sparse-vector
cap.  Each empirical profile has Euclidean sensitivity at most
$\sqrt d/n$.  With the value of $\sigma$ in
Algorithm~\ref{alg:frontier-ladder}, Lemma~\ref{lem:gaussian} makes the
joint Gaussian releases $(\eta_0,\delta_0)$-differentially private and
bounds their $\ell_\infty$ noise by $u$, simultaneously except with
probability $\zeta_0$, provided
\begin{equation*}
    n
    \ge
    C'
    \frac{
      \sqrt{
        dL\log(C'/(\varepsilon\beta))
        \log(C'dL/(\varepsilon\beta))
      }
    }{\varepsilon^2}
\end{equation*}
for a universal constant $C'$.  Coordinatewise clipping is
postprocessing and, since $\widehat R_t(S)\in[0,1]^d$, cannot increase
this distance.  Hence every accepted estimate satisfies
$\|\widetilde R_t-\widehat R_t(S)\|_\infty\le u$.
The square-root term in \eqref{eq:geom-upper} dominates both requirements
after choosing the universal constant $C$ sufficiently large.

By composition, the joint transcript is
$(\frac{\varepsilon}{128},\frac{\varepsilon\beta}{128})$-differentially private, and the two utility events fail with 
probability at most $\varepsilon\beta/64$.  These are the
assumptions of Lemma~\ref{fact:transfer} with
$a=\varepsilon/2$ and $b=\beta/2$.

On this joint utility event, Lemma~\ref{lem:packing} shows that the cap is
not reached, and Lemma~\ref{lem:sample-accuracy} gives
\begin{equation*}
    \max_{t\le k}|\widehat v_t-Q_t(S)|
    \le\frac{\varepsilon}{16}.
\end{equation*}
Lemma~\ref{fact:transfer} therefore yields
\begin{equation*}
    \max_{t\le k}|\widehat v_t-Q_t(P)|
    \le\frac{\varepsilon}{2}
\end{equation*}
with failure probability at most $\beta/2$.  Finally,
Lemma~\ref{lem:bias} and the $\log(2k)$ term in
\eqref{eq:geom-upper} give
\begin{equation*}
    |Q_t(P)-V_t^P(w_t)|\le\frac{\varepsilon}{2}
    \qquad(t\le k).
\end{equation*}
The triangle inequality proves the theorem.
\end{proof}

\subsection{Frontier-value release extension}

Frontier-value validity (Definition~\ref{def:frontier-validity}) requires simultaneous accuracy for every weighting in $\mathcal W$.  Frontier Ladder attains this stronger guarantee by publishing its retained profiles.

\begin{proposition}[Frontier-value release]
\label{prop:frontier-release-upper}
Under the assumptions of Theorem~\ref{thm:geometric-upper}, if
$n$ satisfies \eqref{eq:geom-upper}, then the variant of Frontier Ladder
that receives only $f_t$, performs the same update of $A_t$, and
publishes $A_t$ after every round satisfies, with probability at least
$1-\beta$,
\begin{equation*}
    \max_{t\le k}\sup_{w\in\mathcal W}
    |F_{A_t}(w)-V_t^P(w)|
    \le\varepsilon.
\end{equation*}
Consequently, the right-hand side of \eqref{eq:geom-upper} also
upper-bounds $\Nfront(k,\mathcal W,\varepsilon,\beta)$.
\end{proposition}
\begin{proof}
We apply the transfer theorem to a query measuring the largest error over all weightings.
Publishing $A_t$ preserves the privacy guarantee from Theorem~\ref{thm:geometric-upper}: the set is a deterministic function of the sparse-vector reports and noisy profiles, whose joint transcript that proof controls.

For any data set $X=(X_1,\ldots,X_n)\in\mathcal Z^n$, define
\begin{equation*}
    H_t^X(w)\coloneqq\min_{s\le t}\frac1n\sum_{i=1}^n
    \langle w,\phi(f_s,X_i)\rangle,
    \qquad
    D_t(X)\coloneqq\sup_{w\in\mathcal W}|F_{A_t}(w)-H_t^X(w)|.
\end{equation*}
The joint utility event in Theorem~\ref{thm:geometric-upper} and Lemma~\ref{lem:sample-accuracy} give
\begin{equation}
    D_t(S)\le5u=\frac{\varepsilon}{16}\qquad(t\le k).
    \label{eq:uniform-query-sample-accuracy}
\end{equation}
For a fixed transcript determining $A_t$ and $f_1,\ldots,f_t$, replacing one observation changes each $H_t^X(w)$ by at most $1/n$, so
\begin{equation*}
    |D_t(X)-D_t(X')|
    \le\sup_{w\in\mathcal W}|H_t^X(w)-H_t^{X'}(w)|\le\frac1n.
\end{equation*}
Thus the adaptively chosen queries $D_t$ are $1/n$-sensitive.
Answering each by $0$ preserves privacy and, by \eqref{eq:uniform-query-sample-accuracy}, gives simultaneous sample error at most $\varepsilon/16$ on the same utility event.
Lemma~\ref{fact:transfer}, with $a=\varepsilon/2$ and $b=\beta/2$, therefore yields
\begin{equation}
    \E_{X\sim P^n}D_t(X)\le\frac{\varepsilon}{2}\qquad(t\le k)
    \label{eq:uniform-query-population}
\end{equation}
with probability at least $1-\beta/2$.

On this event, for every $t\le k$ and $w\in\mathcal W$, the triangle inequality and Lemma~\ref{lem:bias} give
\begin{align*}
    |F_{A_t}(w)-V_t^P(w)|
    &\le\E_{X\sim P^n}|F_{A_t}(w)-H_t^X(w)|
      +|\E_{X\sim P^n}H_t^X(w)-V_t^P(w)|\\
    &\le\E_{X\sim P^n}D_t(X)+\sqrt{\frac{\log(2t)}{2n}}
    \le\varepsilon,
\end{align*}
where the last term of \eqref{eq:geom-upper} bounds the square root by $\varepsilon/2$ after adjusting the absolute constant.
\end{proof}

For completeness, there is a nonprivate geometric analogy with streaming
$\varepsilon$-hull methods~\cite{blum2018approximate}.  Those methods maintain a small subset whose convex hull approximates a stream.  Here, retained profiles must be released stably, and the approximation is restricted to the allowed weightings.

\section{Variable accuracy and connections to adaptive estimation}
\label{app:variable-accuracy}

Choosing $q$ of order $\varepsilon^{-1/2}$ in
Lemma~\ref{lem:qary-code} gives a large family of profiles at a margin
proportional to $\varepsilon$.  Combining this capacity bound with
Theorem~\ref{thm:capacity-lower} gives the following dependence on accuracy.

\begin{corollary}[Variable-accuracy simplex lower bound]
\label{cor:variable-accuracy-lower}
There are universal constants $c,\varepsilon_0,\beta_0>0$, with $\varepsilon_0\le c$, such that, for integers $d\ge2$ and $k\ge1$, $0<\varepsilon\le\varepsilon_0$, and $0<\beta\le\beta_0$,
\begin{equation}
    \Npw(k,d,\varepsilon,\beta)
    \ge c\min\!\left\{\sqrt{k},\left(\frac c\varepsilon\right)^{c(d-1)}\right\}.
    \label{eq:variable-accuracy-lower}
\end{equation}
\end{corollary}
\begin{proof}
Let $a,A,\varepsilon_*,\beta_*$ denote the respective constants in
Theorem~\ref{thm:capacity-lower}.  Choose
$\varepsilon_0\le\min\{\varepsilon_*,1/(48A)\}$ and
$\beta_0\le\beta_*$, and put
$q\coloneqq1+\lfloor(48A\varepsilon)^{-1/2}\rfloor\ge2$.
Then $A\varepsilon\le1/[48(q-1)^2]$.  Since capacity is nonincreasing
in its margin, Lemma~\ref{lem:qary-code} gives
\begin{equation*}
    \Mcap_{\Delta_d}(d,A\varepsilon)
    \ge\left\lfloor q^{c_0(d-1)}\right\rfloor.
\end{equation*}
For $k\ge4$, Theorem~\ref{thm:capacity-lower} therefore yields
\begin{equation*}
    \Npw(k,d,\varepsilon,\beta)
    \ge\frac{a}{\sqrt2}\min\!\left\{\sqrt{k},
       \left(\frac{1}{48A\varepsilon}\right)^{c_0(d-1)/4}\right\}.
\end{equation*}
Adjusting constants again proves
\eqref{eq:variable-accuracy-lower} for $k\ge4$.
For $1\le k<4$, the same bound follows from adjusting constants given that the sample size is a positive integer.
\end{proof}

For $k\ge4$, combining Proposition~\ref{thm:log-lower} with Corollary~\ref{cor:variable-accuracy-lower} and adjusting the absolute constants gives
\begin{equation}
    \Npw(k,d,\varepsilon,\beta)
    \ge c\left[\frac{\log k}{\varepsilon^2}+
      \min\!\left\{\sqrt{k},\left(\frac c\varepsilon\right)^{c(d-1)}\right\}\right].
    \label{eq:epsilon-lower-summary}
\end{equation}
The upper bounds are \eqref{eq:geom-upper}, with
$L=\mathsf L_{\mathcal W}(k,\varepsilon)$, and \eqref{eq:generic-upper-gaussian}; one may take the smaller of the two.
For the full simplex, \eqref{eq:sep-packing-upper} gives
$\mathsf L_{\Delt_d}(k,\varepsilon)\le1+\min\{k,(1+\lceil40/\varepsilon\rceil)^d\}$.
Thus Frontier Ladder uses
$\widetilde O(\varepsilon^{-d/2-2})$ samples for fixed $d$ and $\beta$, where logarithmic factors in $k$ and $1/\varepsilon$ are suppressed.
The accuracy exponents in these upper and lower bounds hence remain different.

\paragraph{Relation to adaptive query answering.}
The arbitrary-margin reduction~\eqref{eq:convex-reduction} inside the
proof of Theorem~\ref{thm:capacity-lower} also preserves the link to adaptive
query answering beyond constant query accuracy.  If $d\ge2$ and
$\lfloor2^{c_0(d-1)}\rfloor\ge k$, Lemma~\ref{lem:qary-code} with $q=2$
supplies capacity at least $k$ at margin $\gamma=1/48$.
Thus, for every $0<\varepsilon,\beta<1$, that reduction gives
\begin{equation}
    \Naq(k,192\varepsilon,\beta)
    \le\Npw(k,d,\varepsilon,\beta).
    \label{eq:ada-link}
\end{equation}
Hence lower bounds for adaptive query answering transfer to weighted leaderboards, while leaderboard upper bounds in this dimension range also yield upper bounds for adaptive queries.

For population-only adaptive queries at constant confidence, standard bounds have adaptive terms
$\Omega(\sqrt k/\varepsilon)$ and
$\widetilde O(\sqrt k/\varepsilon^2)$, together with the nonadaptive lower bound
$\Omega(\log k/\varepsilon^2)$, in the nontrivial accuracy range~\cite{steinke2015interactive,bassily2016stability}.
The gap in accuracy dependence therefore carries through our reduction.
The matching result of Lyu and Talwar~\cite{lyu2025fingerprinting} assumes simultaneous sample and population accuracy, a stronger requirement than the population accuracy defining $\Naq$.
Relatedly, an accuracy gap also remains for scalar leaderboards ($d=1$).
Shaky Ladder achieves error $\widetilde O(n^{-2/5})$~\cite{hardt2017climbing}, or sample complexity $\widetilde O(\varepsilon^{-5/2})$.
Frontier Ladder has the same leading accuracy dependence because
$\mathsf L_{\Delt_1}(k,\varepsilon)\le2+\lceil40/\varepsilon\rceil$.
The nonadaptive lower bound in Proposition~\ref{thm:log-lower} has only $\varepsilon^{-2}$ dependence.

\begingroup
\setlength{\textfloatsep}{8pt plus 2pt minus 2pt}
\setlength{\floatsep}{6pt plus 2pt minus 2pt}
\setlength{\intextsep}{6pt plus 2pt minus 2pt}
\captionsetup{skip=4pt}
\captionsetup[subfigure]{skip=3pt}

\section{Controlled weighted-leaderboard experiments}
\label{app:controlled-supplements}

We adapt the boosting attack of Blum and Hardt~\cite{blum2015ladder} to a controlled weighted leaderboard. The analyst combines random predictors that happen to perform well on the reused sample; their majority vote inherits a sample advantage that vanishes on fresh data. Here, different predictors become best under different preassigned weightings.

For a run with $k$ total submissions, each population example consists of a fair binary label $Y$ together with $k-1$ independent fair bits.  Predictor $h_t$ returns the $t$-th bit, so every predictor has population accuracy exactly $1/2$.  On a reused sample $S=\{(X_i,Y_i)\}_{i=1}^n$ of size $n=2000$, write
\begin{equation*}
g_t(X,Y)\coloneqq 2\mathbf{1}\{h_t(X)=Y\}-1,
\qquad
\widehat\mu_t\coloneqq\frac1n\sum_{i=1}^n g_t(X_i,Y_i).
\end{equation*}
Thus $g_t$ is $+1$ when predictor $t$ is correct and $-1$ otherwise,
with population mean zero; $\widehat\mu_t>0$ means that predictor $t$
happened to beat chance on the reused sample.

We place these zero-mean fluctuations into the evaluation profiles used in the lower bound. Proposition~\ref{prop:classification-realization} also realizes this construction with binary classifiers and coordinatewise $0$--$1$ loss.

\paragraph{Feedback from many notions of best.}
To make many of these chance fluctuations observable, we use the separated profiles underlying the frontier-capacity lower bound.  We apply the binary profile map of Lemma~\ref{lem:qary-code} to a finite code, obtaining baseline profiles $b_t$ and weightings $w_t$ such that $b_t$ beats every earlier baseline by a fixed margin under $w_t$. Our chosen family at $d=36$ contains $4096$ certified profile--weighting pairs; Section~\ref{app:controlled-implementation} gives implementation details.

We perturb each baseline profile by the corresponding sample fluctuation,
\begin{equation*}
    \widehat R_t\coloneqq b_t-\rho_0\widehat\mu_t\mathbf 1,
    \qquad \rho_0\coloneqq 1/192.
\end{equation*}
The perturbation is small enough relative to the separation margin that submission $t$ remains best under $w_t$, so the reported best-so-far value reveals $\widehat\mu_t$.  Both the predictors and the weightings are fixed before the reused sample is observed.

After the first $k-1$ submissions, the analyst makes its only adaptive submission:
the majority vote of the predictors with $\widehat\mu_t>0$, evaluated using the
remaining profile--weight pair $(b_k,w_k)$ under the same construction above.
These predictors were selected because they happened to agree unusually often with the reused labels.  On a fresh example, however, the label remains an independent fair bit, so this advantage disappears.

\paragraph{Scalar leaderboard comparison.}
We compare this with an ordinary scalar leaderboard using the same random predictors, reused samples, and submission order.  To put a given fluctuation $\widehat\mu_t$ on the same numerical scale as in the changing-weighting interaction, we also use the same factor $\rho_0$, defining the per-example loss
$\ell_t(X,Y)\coloneqq\frac12-\rho_0 g_t(X,Y)$.
This is simply an affine rescaling of $0$--$1$ loss.  Every predictor has population loss $1/2$, while its reused-sample loss is $1/2-\rho_0\widehat\mu_t$.  Because the scalar leaderboard reports only the best value seen so far,
its numerical feedback changes only when $\widehat\mu_t$ sets a new
record.  In our matched scalar attack, the final submission combines
these record-setting predictors.

\paragraph{Results.}
For both interactions, we report the weighted-leaderboard validity error
\begin{equation*}
    E_k\coloneqq\max_{t\le k}\bigl|\widehat v_t-V_t^P(w_t)\bigr|,
\end{equation*}
where, as in Section~\ref{sec:weighted-leaderboard}, $\widehat v_t$ is the reported value and $V_t^P(w_t)$ is the population value of the best submission so far under the requested weighting; the scalar interaction uses the same fixed weighting at every round.

Figure~\ref{fig:controlled-attack} shows that the two interactions increasingly separate as the submission budget grows.  A single weighting hides most chance fluctuations because only new records affect the transcript; changing weightings make many more of the same fluctuations observable.  At the largest budget, this produces more than a fourfold gap in the $90$th-percentile validity error.

\paragraph{Dimension transition.}
The same construction lets us vary how many distinct fluctuations can be exposed.  We fix $n=2000$, $k=4096$, and $\rho_0=1/192$, and repeat the changing-weighting attack while varying $d$.  At each dimension, we use a finite family of certified profile--weight pairs, capped at the submission budget. This comparison measures the performance of these families. Increasing $d$ does not make any random predictor better; larger families expose more chance fluctuations.

The two right-hand panels of Figure~\ref{fig:controlled-attack} show the resulting transition.  The smaller tested families expose fewer fluctuations and produce weaker attacks. As the certified family grows, so does the validity error. In our finite construction, $d=36$ is the first tested dimension at which all 4095 candidate fluctuations can be exposed, leaving the final submission for their adaptive majority. At larger tested dimensions, we embed this same full family, so these runs remain limited by the submission budget.

\begin{figure}[t]
    \centering
    \makebox[\linewidth][c]{%
        \resizebox{0.96\linewidth}{!}{%
        \begin{minipage}{1.04\linewidth}
            \begin{minipage}[c]{0.39\linewidth}
                \centering
                \resizebox{\linewidth}{!}{%
                    \input{fig/tex/figure_main_clean}%
                }
            \end{minipage}%
            \hfill
            \begin{minipage}[c]{0.59\linewidth}
                \centering
                \begin{tikzpicture}
                    \node[anchor=south west,inner sep=0] (img) at (0,0)
                      {\includegraphics[width=\linewidth]{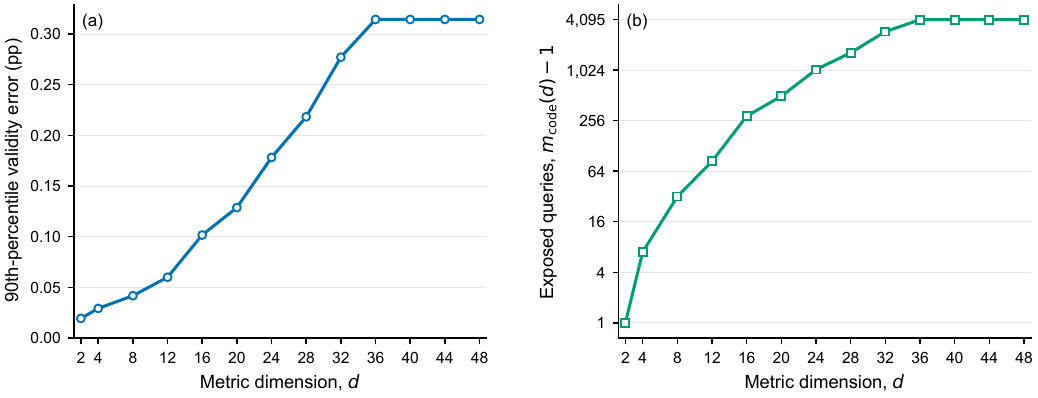}};
                    \fill[white] ($(img.north west)+(0.75cm,-0.02cm)$)
                      rectangle ($(img.north west)+(1.10cm,-0.30cm)$);
                    \fill[white] ($(img.north west)+(6.02cm,-0.02cm)$)
                      rectangle ($(img.north west)+(6.34cm,-0.30cm)$);
                \end{tikzpicture}
            \end{minipage}%
        \end{minipage}%
        }%
    }

    \caption{\textbf{Boosting through many notions of best.}
    Left: changing preassigned weightings expose more chance fluctuations and
    produce larger validity error than one fixed weighting. Right: for the tested
    code families, increasing the number of criteria at fixed $n=2000$ and $k=4096$ increases both error
    and exposure, until the submission budget becomes the limit. All runs use
    exact empirical feedback, paired samples, and $\rho_0=1/192$; validity
    curves show $90$th percentiles over $1000$ paired repetitions.}
    \label{fig:controlled-attack}
\end{figure}

\paragraph{Implementation details.}\label{app:controlled-implementation}
The finite families use a fixed-seed greedy binary packing with at most $100{,}000$ candidates per dimension. At $d=36$, it yields $4096$ vectors of length $35$ with minimum Hamming distance $9$. Larger tested dimensions append constant profile coordinates with zero weight. Paired repetitions share the sample and random predictors.

\section{Real-benchmark protocol and additional results}
\label{app:supp-experiments}
\label{sec:llm-supplement}

\subsection{Calibration and evaluation protocol}

Table~\ref{tab:benchmark-splits} gives the numbers of criteria and split sizes for the four historical snapshots. We split within tasks, keeping duplicate inputs together, and average category losses equally on $C$, $S$, and $T$. Within each category, HELM Capabilities and Open LLM average item losses, while LiveBench and HELM Lite average task means. We choose endpoints using $C$ alone and fix submission budgets before the interaction.

\begin{table}[t]
\centering
\small
\begin{tabular}{@{}lrrrr@{}}
\toprule
Benchmark & $d$ & $|C|$ & $|S|$ & $|T|$ \\
\midrule
HELM Capabilities & 5 & 797 & 1,064 & 2,126 \\
LiveBench & 6 & 156 & 204 & 408 \\
HELM Lite & 10 & 2,041 & 2,721 & 5,438 \\
Open LLM & 6 & 960 & 1,278 & 2,562 \\
\bottomrule
\end{tabular}
\caption{\textbf{Numbers of criteria and split sizes.}
Counts refer to examples in the pinned evaluation frames and are fixed across trials.  The calibration examples remain fixed; the remaining examples are repartitioned into $S$ and $T$.}
\label{tab:benchmark-splits}
\end{table}

The genuine-model pools contain $68$, $33$, $14$, and $7$ models, respectively.  HELM Capabilities includes every model with complete aligned scores on its five scenarios.  LiveBench uses $33$ endpoints with one published judgment per item from the October 22, 2024 snapshot.  For HELM Lite, we include models with complete aligned scores on all ten scenarios and compatible prompt templates.  Open LLM uses the seven archived candidate endpoints with complete scores on the six retained tasks.

Table~\ref{tab:calibration-pairs} gives the selected endpoints. Eligible pairs have both endpoints on the calibration Pareto frontier, and an itemwise oracle choosing the better endpoint must beat the best genuine calibration model. We call this improvement \emph{oracle headroom}. For each pair, \emph{joint Pareto survival} is the fraction of $2000$ calibration bootstrap resamples in which both endpoints remain nondominated, resampling duplicate-input groups within tasks. Among pairs within one percentage point of the highest survival, we maximize mean absolute score difference, then the smaller directional item-win rate; remaining ties use oracle headroom and model-name order.

\begin{table}[t]
\centering
\scriptsize
\setlength{\tabcolsep}{3pt}
\begin{tabular}{@{}lp{0.21\linewidth}p{0.22\linewidth}rrrrr@{}}
\toprule
Benchmark & Endpoint A & Endpoint B & Rank A/B & Headroom & Survival & Disagree & Both win \\
& & & & (pp) & & & \\
\midrule
HELM Capabilities & \shortstack[l]{\texttt{xai/}\texttt{grok-4-0709}} & \shortstack[l]{\texttt{openai/}\texttt{o3-2025-04-16}} & $7/3$ & 5.800 & 0.994 & 0.154 & 0.109 \\
\addlinespace[2pt]
LiveBench & \shortstack[l]{\texttt{gemini-1.5-pro-exp-}\texttt{0801}} & \shortstack[l]{\texttt{claude-3-5-}\\\texttt{sonnet-20240620}} & $6/1$ & 10.373 & 0.990 & 0.267 & 0.139 \\
\addlinespace[2pt]
HELM Lite & \shortstack[l]{\texttt{meta/}\texttt{llama-3.3-70b-}\\\texttt{instruct-turbo}} & \shortstack[l]{\texttt{google/}\texttt{gemini-1.5-pro-002}} & $4/3$ & 5.316 & 0.996 & 0.153 & 0.131 \\
\addlinespace[2pt]
Open LLM & \shortstack[l]{\texttt{mistralai/}\texttt{Mistral-7B-}\\\texttt{Instruct-v0.2}} & \shortstack[l]{\texttt{microsoft/}\texttt{Phi-3-mini-4k-}\\\texttt{instruct}} & $4/1$ & 7.708 & 0.999 & 0.233 & 0.077 \\
\bottomrule
\end{tabular}
\caption{\textbf{Calibration-selected endpoint pairs.}
Rank uses the official aggregate over the full benchmark frame.
``Disagree'' is the mean absolute item-score difference; ``Both win'' is the
smaller directional item-win rate, treating differences at most $10^{-12}$ as
ties. Both use the benchmark's task and category aggregation. Headroom and
survival are defined in the text.}
\label{tab:calibration-pairs}
\end{table}

We submit genuine models in model-name order, followed by $A$, $B$, the random routers, and the final majority router. Initialization uses the same disclosure rules, and the Pareto analyst remembers any endpoint profile previously released. The attack budget $k$ includes both endpoints, $k-3$ random routers, and the final router. Routing rules give the same decision on repeated prompts; majority ties, including an empty retained set, use a fixed random routing rule.

For feedback, attack selection, and ranks, scores are compared to ten decimal places.

Each condition uses $250$ trials that repartition the noncalibration examples into $S$ and $T$, holding $C$, the endpoints, and the router bank fixed. Pointwise $95\%$ percentile intervals for leaderboard gaps use $2000$ paired bootstrap resamples. They describe partition variability conditional on these fixed choices.

\subsection{Additional results and diagnostics}
\label{app:additional-results}

\paragraph{Validity and feedback.}

\begin{figure}[t]
    \centering
    \begin{subfigure}[t]{0.46\linewidth}
        \centering
        \includegraphics[width=0.95\linewidth]{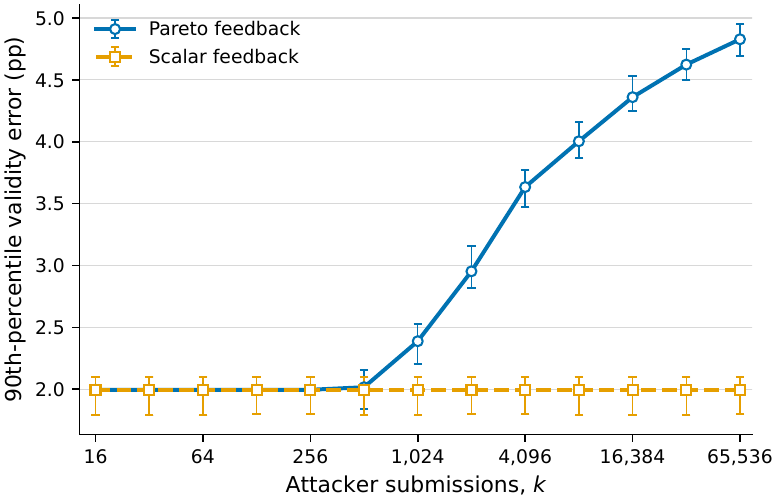}
        \caption{HELM Lite}
    \end{subfigure}\hspace{0.02\linewidth}
    \begin{subfigure}[t]{0.46\linewidth}
        \centering
        \includegraphics[width=0.95\linewidth]{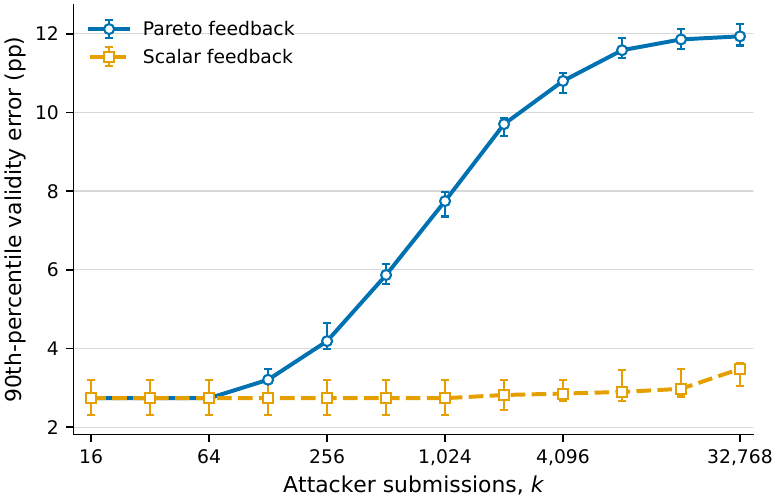}
        \caption{Open LLM}
    \end{subfigure}
    \caption{\textbf{Additional validity results.}
    The curves show the $90$th-percentile reused-to-held-out validity error across $250$ paired trials under each benchmark's common equal-weight objective.  Bars are paired-bootstrap $95\%$ confidence intervals.}
    \label{fig:additional-preloaded-validity}
\end{figure}

Figure~\ref{fig:additional-preloaded-validity} extends the leaderboard-gap comparisons in the top row of Figure~\ref{fig:real-budget-results} to HELM Lite and Open LLM. Table~\ref{tab:four-benchmark-summary} gives paired intervals for the Pareto-minus-scalar difference and the fraction of trials in which both endpoint profiles are available to the Pareto analyst. At the headline budgets, all four estimated differences and their paired intervals are positive. Table~\ref{tab:full-rank-diagnostics} reports mean ranks and candidate-router feedback counts.

\begin{table}[t]
\centering
\small
\begin{tabular}{@{}lrr@{}}
\toprule
Benchmark & $\Prb(A,B\text{ visible on }S)$ & Pareto--scalar gap (pp) [95\% CI] \\
\midrule
HELM Capabilities & 0.932 & 2.503 [2.207, 2.918] \\
LiveBench & 0.884 & 3.320 [2.405, 3.899] \\
HELM Lite & 0.992 & 2.629 [2.476, 2.843] \\
Open LLM & 1.000 & 8.453 [8.217, 8.959] \\
\bottomrule
\end{tabular}
\caption{\textbf{Endpoint visibility and paired error differences.}
At the budgets in Table~\ref{tab:llm-summary}, visibility records whether both
endpoint profiles have been released before the first random router. Differences
between the $90$th-percentile leaderboard gaps and their paired bootstrap
intervals are computed before rounding.}
\label{tab:four-benchmark-summary}
\end{table}

\paragraph{Final-router gaps and ranks.}

Figure~\ref{fig:final-router-gap} compares the final router's own aggregate loss on $S$ and $T$. The majority assembled from Pareto-visible rules has a larger gap over much of the submission range.

\begin{figure}[p]
    \centering
    \begin{subfigure}[t]{0.46\linewidth}
        \centering
        \includegraphics[width=0.95\linewidth]{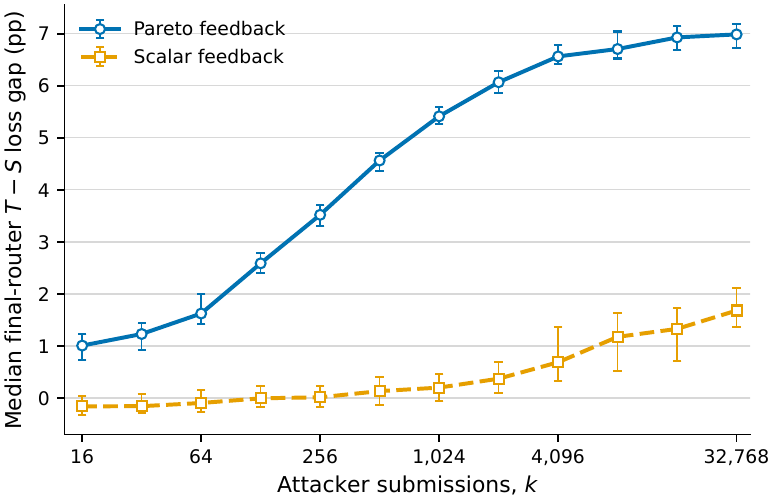}
        \caption{HELM Capabilities}
    \end{subfigure}\hspace{0.02\linewidth}
    \begin{subfigure}[t]{0.46\linewidth}
        \centering
        \includegraphics[width=0.95\linewidth]{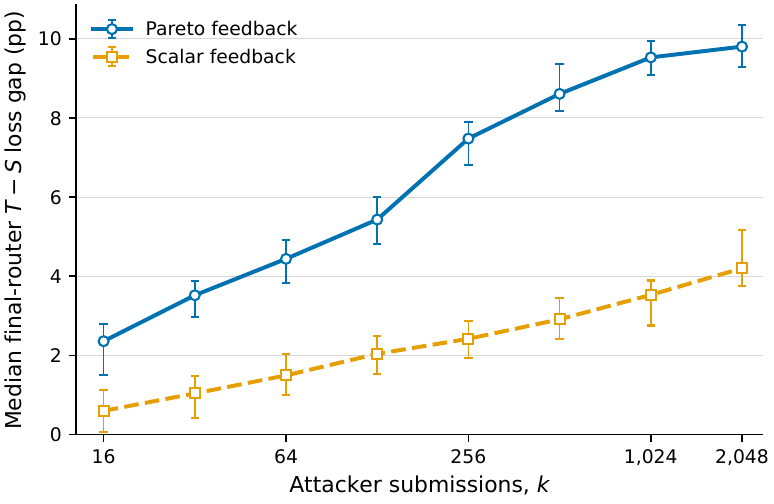}
        \caption{LiveBench}
    \end{subfigure}

    \par\smallskip

    \begin{subfigure}[t]{0.46\linewidth}
        \centering
        \includegraphics[width=0.95\linewidth]{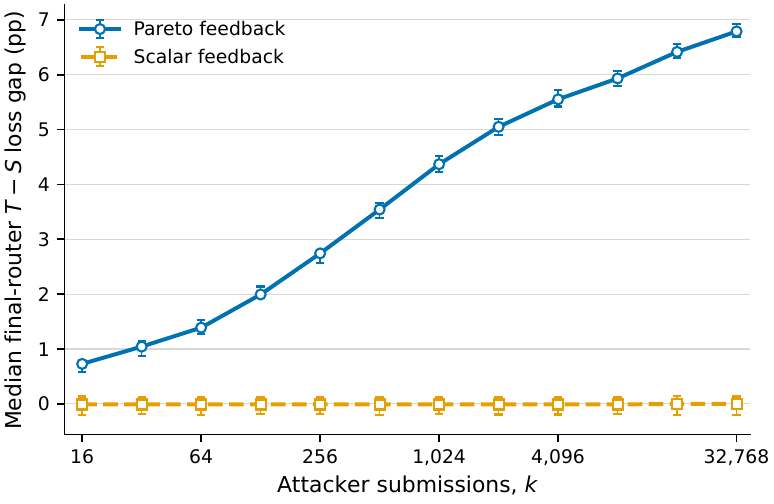}
        \caption{HELM Lite}
    \end{subfigure}\hspace{0.02\linewidth}
    \begin{subfigure}[t]{0.46\linewidth}
        \centering
        \includegraphics[width=0.95\linewidth]{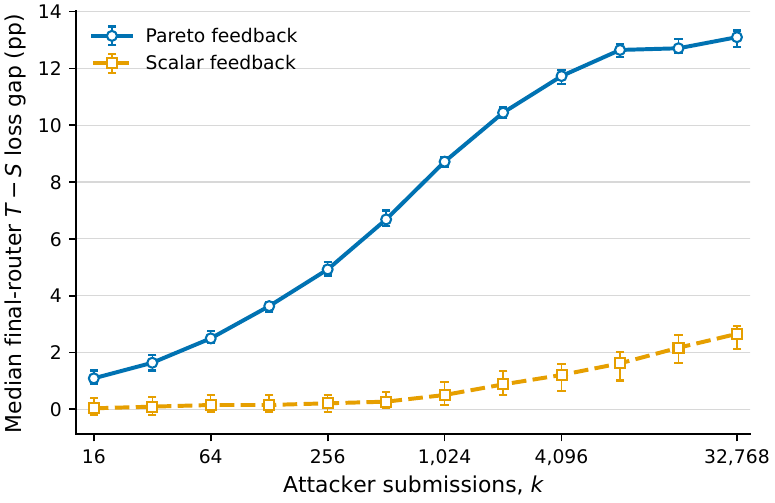}
        \caption{Open LLM}
    \end{subfigure}
    \caption{\textbf{Held-out minus reused loss of the final adaptive router.}
    Curves show the final router's $T-S$ equal-weight loss gap across $250$ paired trials.  The final router is fixed before $T$ is opened.}
    \label{fig:final-router-gap}

\par\medskip
\begingroup
\captionsetup{type=table}
\centering
\small
\setlength{\tabcolsep}{5pt}
\begin{tabular}{@{}llrll@{}}
\toprule
Benchmark & Feedback & \shortstack{Mean feedback\\routers}
& \shortstack{Mean rank on $S$\\{}[95\% CI]}
& \shortstack{Mean rank on $T$\\{}[95\% CI]} \\
\midrule
HELM Capabilities & Pareto & 1,018.156 & 1.34 [1.19, 1.52] & 5.67 [5.55, 5.78] \\
                  & Scalar & 2.204 & 3.18 [2.86, 3.51] & 5.68 [5.57, 5.79] \\
\addlinespace[2pt]
LiveBench & Pareto & 528.732 & 1.18 [1.12, 1.25] & 2.62 [2.54, 2.71] \\
          & Scalar & 2.324 & 1.38 [1.29, 1.48] & 2.52 [2.44, 2.60] \\
\addlinespace[2pt]
HELM Lite & Pareto & 12,964.444 & 1.02 [1.00, 1.06] & 4.06 [4.03, 4.09] \\
          & Scalar & 0.008 & 4.16 [4.09, 4.23] & 4.04 [4.02, 4.06] \\
\addlinespace[2pt]
Open LLM & Pareto & 1,787.376 & 1.00 [1.00, 1.00] & 3.20 [3.15, 3.24] \\
         & Scalar & 1.784 & 1.80 [1.66, 1.92] & 3.08 [3.05, 3.12] \\
\bottomrule
\end{tabular}
\caption{\textbf{Final-router ranks and candidate-router feedback.}
Means and percentile-bootstrap rank intervals use the $250$ paired trials at
the budgets in the caption of Table~\ref{tab:llm-summary}. Feedback counts
include only random candidates whose profiles or aggregate records are
released, excluding genuine models, endpoints, and the final router.
A released Pareto profile need not be retained for the final majority.}
\label{tab:full-rank-diagnostics}
\endgroup
\end{figure}

\begin{figure}[!t]
    \centering
    \includegraphics[width=0.95\linewidth]{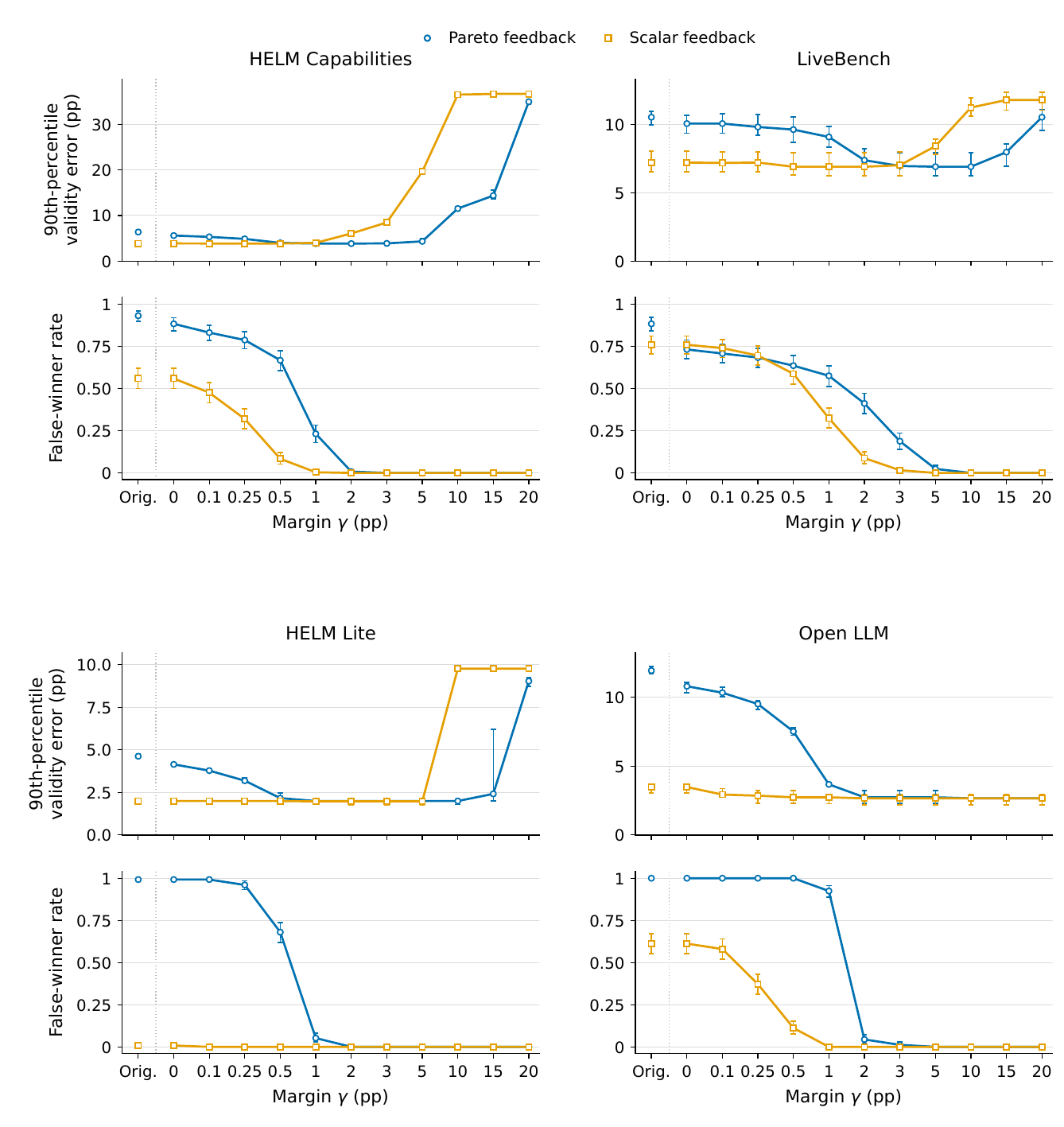}
    \caption{\textbf{Results with larger exposure margins.}
    Figure~\ref{fig:llm-feedback} extended to $\gamma=20$ percentage points
    and all four benchmarks. The Pareto validity error also rises at
    large margins on HELM Capabilities, LiveBench, and HELM Lite as useful
    updates are withheld. Rows show $90$th-percentile published-value error
    and false-winner rates (first on $S$, but not on $T$).
    Each benchmark uses $250$ paired trials; attack budgets are $32{,}768$
    except on LiveBench ($2{,}048$).
    Bars are pointwise $95\%$ bootstrap intervals, using $2000$ resamples for
    errors and $10{,}000$ for false-winner rates. At $\gamma=0$, positive convex exposure is already
    stricter than the original Pareto nondominance rule.}
    \label{fig:margin-gates-extended}
\end{figure}

\begin{samepage}
Table~\ref{tab:full-rank-diagnostics} supplies mean ranks at the headline budgets. Competition rank against genuine models is
$\operatorname{rank}_D(G)\coloneqq
1+\#\{M\text{ genuine}:L_D(M)<L_D(G)\}$, $D\in\{S,T\}$.
Ties share rank. Mean-rank intervals use $10{,}000$ bootstrap resamples of
the $250$ trials, keeping $S$ and $T$ paired within each feedback condition.
For false-winner intervals in Figures~\ref{fig:real-budget-results},
\ref{fig:llm-feedback}, and~\ref{fig:margin-gates-extended}, we bootstrap the
per-trial indicator $10{,}000$ times, pairing the resampled trial IDs across
feedback conditions. No final router ranks first on $T$ at any tested budget
or margin, so false-winner rates equal first-place frequencies on $S$.
\par\end{samepage}

\clearpage
\endgroup

\end{document}